\documentclass[lettersize,journal]{IEEEtran}

\usepackage[colorlinks,urlcolor=blue,linkcolor=blue,citecolor=blue]{hyperref}
\usepackage{amsmath,amsfonts}
\usepackage{amsthm}
\usepackage{array}
\usepackage{textcomp}
\usepackage{stfloats}
\usepackage{url}
\usepackage{verbatim}
\usepackage{graphicx}
\usepackage{cite}
\usepackage{booktabs}
\usepackage{multirow}
\usepackage{threeparttable}
\usepackage{dcolumn}
\usepackage{bm}
\usepackage{array}
\usepackage{enumerate}
\usepackage{siunitx}
\usepackage{subcaption}
\usepackage{caption}
\usepackage{float}
\usepackage[ruled,linesnumbered]{algorithm2e}
\newtheorem{theorem}{Theorem}

\begin{document}

\title{Oracle-Guided Masked Contrastive Reinforcement Learning for Visuomotor Policies}

\author{Yuhang Zhang, \IEEEmembership{Student Member, IEEE}, Jiaping Xiao, \IEEEmembership{Student Member, IEEE}, Chao Yan, \IEEEmembership{Member, IEEE}, and \\ Mir Feroskhan, \IEEEmembership{Member, IEEE}
        
\thanks{Y. Zhang, J. Xiao, and M. Feroskhan are with the School of Mechanical and Aerospace Engineering, Nanyang Technological University, Singapore 639798, Singapore (e-mail: yuhang004@e.ntu.edu.sg; jiaping001@e.ntu.edu.sg; mir.feroskhan@ntu.edu.sg). C. Yan is with the College of Automation Engineering, Nanjing University of Aeronautics and Astronautics, Nanjing, 211106, China (e-mail: yanchao@nuaa.edu.cn). \textit{(Corresponding authors: Mir Feroskhan and Chao Yan.)}

Project materials and supplementary information are available at: \url{https://zzzzzyh111.github.io/OMC-RL-Website/}.}
}



\maketitle
 
\begin{abstract}
A prevailing approach for learning visuomotor policies is to employ reinforcement learning to map high-dimensional visual observations directly to action commands. However, the combination of high-dimensional visual inputs and agile maneuver outputs leads to long-standing challenges, including low sample efficiency and significant sim-to-real gaps. To address these issues, we propose Oracle-Guided Masked Contrastive Reinforcement Learning (OMC-RL), a novel framework designed to improve the sample efficiency and asymptotic performance of visuomotor policy learning. OMC-RL explicitly decouples the learning process into two stages: an upstream representation learning stage and a downstream policy learning stage. In the upstream stage, a masked Transformer module is trained with temporal modeling and contrastive learning to extract temporally-aware and task-relevant representations from sequential visual inputs. After training, the learned encoder is frozen and used to extract visual representations from consecutive frames, while the Transformer module is discarded. In the downstream stage, an oracle teacher policy with privileged access to global state information supervises the agent during early training to provide informative guidance and accelerate early policy learning. This guidance is gradually reduced to allow independent exploration as training progresses. Extensive experiments in simulated and real-world environments demonstrate that OMC-RL achieves superior sample efficiency and asymptotic policy performance, while also improving generalization across diverse and perceptually complex scenarios.
\end{abstract}


\section{Introduction}

\begin{figure}[t!]
    \centering
    
    \begin{subfigure}{\linewidth}
        \centering
        \includegraphics[width=1.0\linewidth]{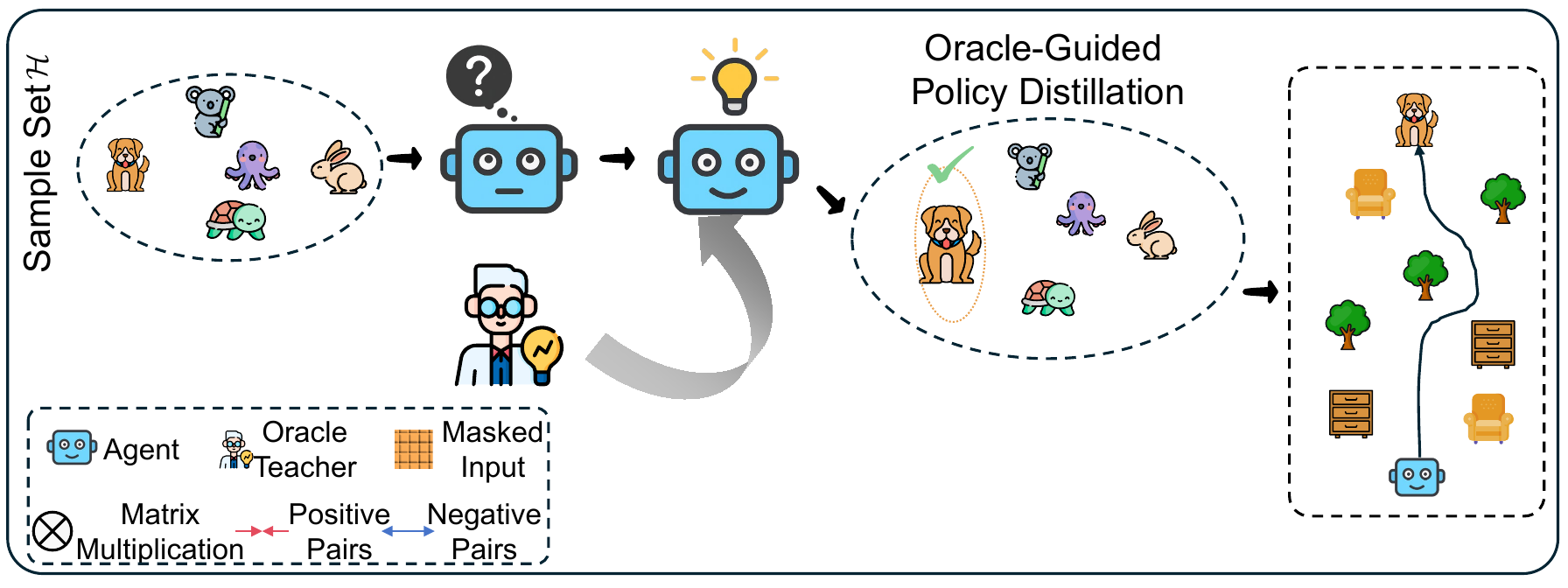}
        \caption{Vanilla Teacher-Student Policy Learning}
        \label{fig:a}
    \end{subfigure}
    
    \begin{subfigure}{\linewidth}
        \centering
        \includegraphics[width=1.0\linewidth]{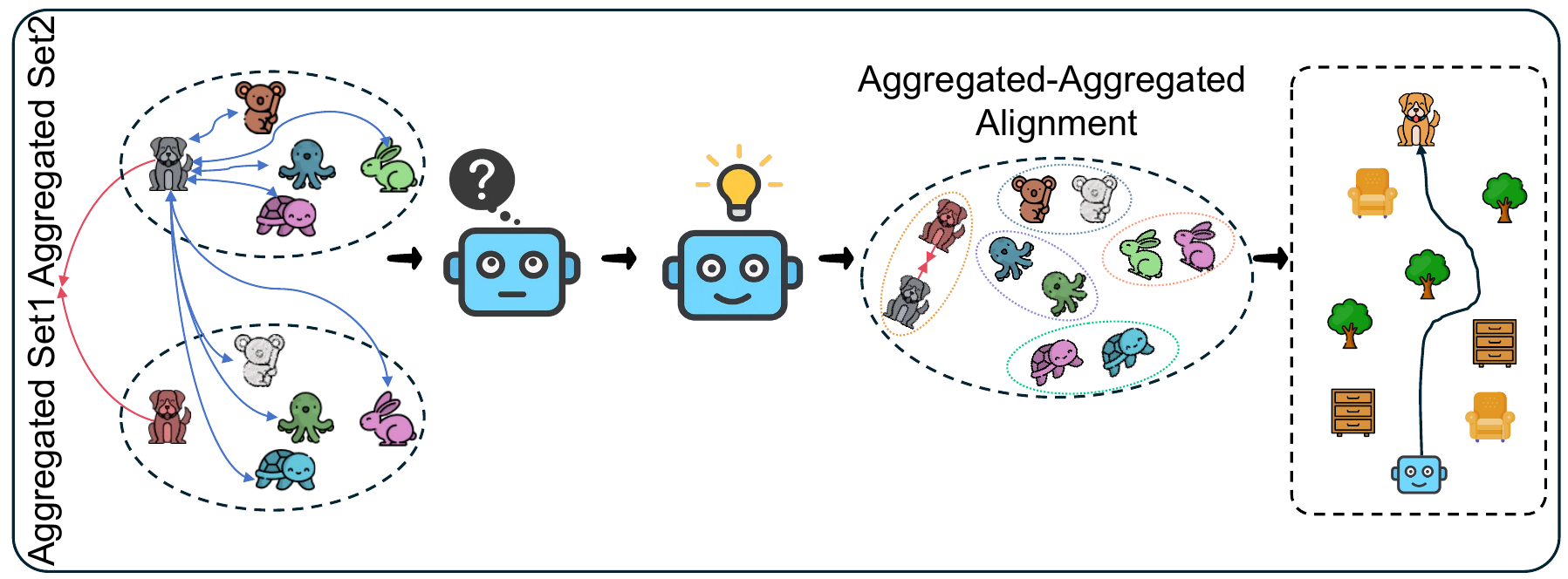}
        \caption{Vanilla Contrastive Learning}
        \label{fig:b}
    \end{subfigure}
    
    \begin{subfigure}{\linewidth}
        \centering
        \includegraphics[width=1.0\linewidth]{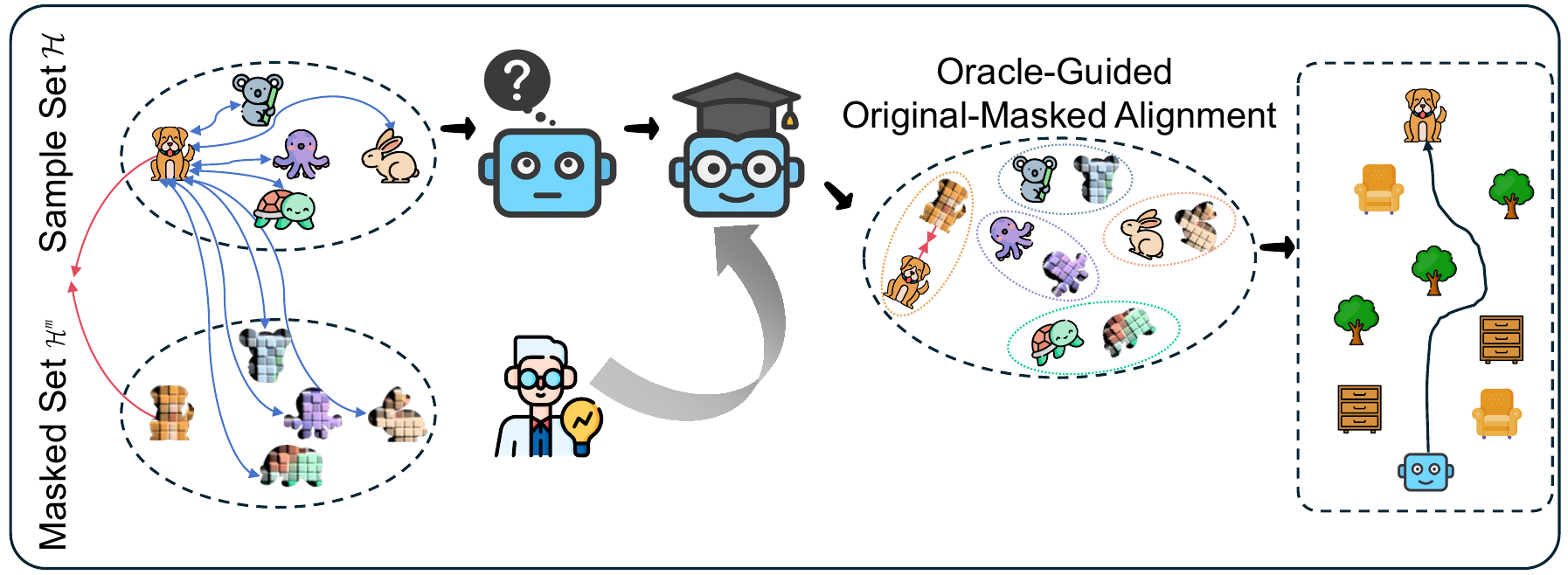}
        \caption{Oracle-Guided Masked Contrastive Learning}
        \label{fig:c}
    \end{subfigure}
    
    \caption{
Illustration of different learning paradigms for goal-conditioned representation and policy learning.  
(a) In vanilla teacher-student policy learning, the agent learns from an oracle teacher via policy distillation to better identify and reach the target through supervised imitation.  
(b) In vanilla contrastive learning, given an observation set $\mathcal{H}$, positive pairs (red arrows) formed between aggregated features are aligned, while all other negatives (blue arrows) are pushed apart.  
(c) Oracle-guided masked contrastive learning (ours) combines masked contrastive representation learning with oracle-supervised policy learning to jointly improve both feature encoding and downstream decision-making.}

    \label{fig:mask}
\end{figure}

First-person-view navigation for autonomous robots represents a challenging task that emulates human behavioral paradigms \cite{wu2023human}, with wide-ranging applications across critical domains, including autonomous driving \cite{he2023fear}, field exploration \cite{ye2024ood}, and urban surveillance \cite{xiao2025vision}. Skilled human operators demonstrate the ability to perform agile and precise maneuvers using only delayed and occasionally occluded visual feedback from onboard monocular cameras, without compromising safety. This remarkable capability is attributed to the human brain’s ability to distill abstract perceptual signals into compact, low-dimensional representations that serve as internal interpretations for guiding action. Given the complexity and unstructured nature of real-world environments, autonomous robotic systems must be equipped with a feature extractor capable of filtering out noise and redundancy from high-dimensional visual feeds, and extracting effective representations for decision-making.



Traditional methods typically follow a mapping–planning paradigm \cite{jin2024gs}, including Simultaneous Localization and Mapping (SLAM) \cite{xu2025airslam} and Structure from Motion (SfM) \cite{wang2024tc}. These pipelines heavily rely on sophisticated hand-crafted features for keypoint detection and matching. Although effective, their reliance on feature engineering limits flexibility and incurs high computational cost. Moreover, their performance degrades significantly in textureless or visually ambiguous environments. Deep learning (DL) \cite{he2016deep} offers an alternative paradigm by leveraging deep neural networks (DNNs) to extract task-relevant features directly from raw sensory observations and map them to action commands. In contrast to heuristic feature engineering, DNN-based encoders extract representations directly from data, allowing the model to adapt to diverse input distributions without manual intervention. However, training visuomotor policies for autonomous robot navigation with DL often demands expert-annotated action datasets as supervision signals, which are expensive to collect, particularly in real-world settings \cite{loquercio2018dronet}.

Recently, reinforcement learning (RL) has achieved expert-level performance, even exceeding human professionals in domains such as Go \cite{silver2018general}, video games \cite{vinyals2019grandmaster}, and Mahjong \cite{li2020suphx}. By integrating RL algorithms with high-capacity DNNs, deep reinforcement learning (DRL) is increasingly adopted to address complex robotic perception and control tasks \cite{kaufmann2023champion, xiao2024learning, yan2025selective}. Specifically, in the domain of autonomous robot navigation where visuomotor policies take a sequence of image frames as input, existing DRL algorithms suffer from two major limitations: (i) \textbf{the curse of dimensionality}. The input data is extremely high-dimensional, making it difficult for the policy network to directly learn effective mappings from observations to actions, which often leads to unstable learning dynamics and poor generalization; (ii) \textbf{sample inefficiency}. Due to the trial-and-error nature of RL and the complexity of visual navigation tasks, training typically requires a large number of rollout trajectories, reflecting low sample efficiency and incurring substantial computational cost. These limitations jointly affect both the efficiency and effectiveness of DRL algorithms. 

To address the aforementioned challenges, recent studies have explored representation learning, which trains a feature encoder to compress high-dimensional visual inputs into compact representations before feeding them into the downstream policy network. However, it remains non-trivial to design representation learning methods that are both effective and sample-efficient in DRL settings. Existing approaches generally fall into two categories: one category introduces auxiliary self-supervised losses to guide encoder learning \cite{yarats2021improving, yarats2021reinforcement, yoo2024mono}, and the other learns a world model that captures transition dynamics and enables planning in the latent space \cite{zhang2019solar, lee2020stochastic, zhao2023learning}.

As a prominent self-supervised strategy, contrastive learning has recently emerged as a powerful tool in representation learning \cite{chen2020simple, jiang2021improving, zhang2022semi}. Unlike class-level supervised learning approaches, it operates at the instance level and is primarily designed to learn generalizable representations through pretext tasks. Specifically, during pre-training, each input image is subjected to data augmentation techniques such as random cropping or color jittering, resulting in two distinct views of the same instance. These views are then passed through a shared feature encoder to produce latent representations. The fundamental objective of contrastive learning is to draw representations of the same instance closer in the feature space while enforcing dissimilarity with the rest in the batch, typically enforced via a contrastive loss \cite{chen2020simple}. Encouraged by its notable performance in the computer vision domain, contrastive learning has been extended to RL settings, commonly as an auxiliary loss \cite{yarats2021improving, laskin2020curl}. A prominent example is CURL \cite{laskin2020curl}, which stands for contrastive unsupervised representations for RL. It demonstrates that pixel-based RL can rival the performance of state-based RL counterparts. 

Further explorations have applied contrastive learning to autonomous robot navigation tasks \cite{zhang2025learning, xing2024contrastive, fu2023learning}, where it has shown promise in improving robustness to visual interference and enhancing sample efficiency. Nonetheless, most existing approaches neglect the strong temporal correlations across consecutive frames. In visuomotor policy learning, understanding such temporal coherence is crucial. Intuitively, a representation encoder should not only process the current observation accurately, but also encode contextual relationships across adjacent frames. In addition, jointly optimizing the encoder and policy network often leads to domain-specific overfitting, thereby hindering generalization to unseen environments.


Encouraged by prior advances, this paper proposes a novel framework named oracle-guided masked contrastive RL (OMC-RL) to enhance the sample efficiency and asymptotic performance of visuomotor policy learning. As illustrated in Fig. \ref{fig:mask}, OMC-RL differs from vanilla teacher-student policy learning and contrastive learning by integrating masked temporal contrastive learning with oracle-guided supervision, leading to joint improvements in both feature encoding and downstream decision-making. Specifically, we decouple the feature encoder from the RL policy to allow stable representation learning and independent control over encoder optimization. For upstream feature extraction, unlike previous approaches that utilize DNNs as pixel encoders \cite{laskin2020curl}, we incorporate a Transformer encoder \cite{vaswani2017attention} as an auxiliary module to model the temporal correlation across sequential observations. Specifically, the latent representations of certain frames are randomly masked and subsequently reconstructed by the Transformer module, encouraging the model to capture temporally consistent representations. A contrastive loss is subsequently designed to jointly train both the DNN encoder and Transformer, encouraging the reconstructed features to be similar to the unmasked ones while distinct from others. For downstream policy learning, we introduce an oracle teacher with privileged access to complete environment states, which is pretrained to generate expert actions. A Kullback–Leibler (KL) divergence loss is designed between the oracle and agent policies to guide policy learning. We conduct extensive experiments in simulated and real-world environments to validate the effectiveness of OMC-RL. Simulation experiments show that OMC-RL consistently outperforms all learning- and planning-based baselines across diverse scenarios. Moreover, real-world in-flight experiments demonstrate that OMC-RL outperforms the state-of-the-art baseline, further validating the advantages of the proposed framework in practical deployment.


The contributions of this paper are summarized below.
\begin{enumerate}
    \item We propose a masked contrastive learning method for robot navigation to address the curse of dimensionality and sample inefficiency arising from high-dimensional visual inputs. The encoder is decoupled from the policy network to enable stable and task-agnostic representation learning. By reconstructing masked latent features from sequential inputs, the method improves temporal awareness and enhances robustness to domain shift.
    \item We introduce an oracle-guided visuomotor policy learning framework, where an expert policy with privileged access to global environment states serves as a training-time oracle. By developing a learning-by-cheating strategy based on KL divergence, the agent is empowered to leverage dense supervision from the oracle, thereby improving sample efficiency and asymptotic performance.
    \item We demonstrate that OMC-RL achieves strong simulation performance compared with various baselines. More importantly, we show that it surpasses the state-of-the-art baseline in real-world in-flight experiments and maintains robust performance under diverse visual disturbances.
\end{enumerate}

The remainder of this paper is organized as follows. Section \uppercase\expandafter{\romannumeral2} reviews prior work on visuomotor policy learning, representation learning for RL, and sample-efficient RL. Section \uppercase\expandafter{\romannumeral3} introduces the preliminaries and formalizes the problem. Section \uppercase\expandafter{\romannumeral4} details the methodology of the proposed framework. Section \uppercase\expandafter{\romannumeral5} describes the experimental setup for evaluating the proposed method. Section \uppercase\expandafter{\romannumeral6} presents the results and analysis in simulation and physical environments. Finally, Section \uppercase\expandafter{\romannumeral7} concludes the paper and outlines future research directions.

\section{Related Work}
This section reviews related work in three key areas: visuomotor policy learning from high-dimensional observations, representation learning for RL, and methods for improving sample efficiency in RL.

\subsection{Visuomotor Policy Learning}
At the core of visuomotor policy learning lies the challenge of approximating the mapping from high-dimensional visual observations to control actions. This problem has been widely addressed using DL \cite{loquercio2018dronet, gandhi2017learning, kouris2018learning}, where large-scale real-world datasets, such as crash records \cite{gandhi2017learning} or urban traffic footage \cite{loquercio2018dronet}, are used to train a DNN to perform this mapping from pixel-level inputs to action commands. However, these DL-based approaches suffer from limited generalization capacity and brittleness to distributional shifts, making them highly sensitive to unseen scenarios.

To overcome these drawbacks, DRL \cite{kaufmann2023champion, zhang2025learning1, zhang2024npe, xiao2023collaborative, yan2023collision} has emerged as an alternative, offering a trial-and-error learning paradigm that enables agents to generalize through interaction. Thanks to its exploratory nature and ability to learn from feedback rewards, DRL exhibits greater robustness when deployed in previously unseen environments. For instance, Kaufmann et al. \cite{kaufmann2023champion} proposed a DRL-based drone racing framework that learns agile visuomotor control policies directly from raw pixel input, achieving performance surpassing champions. Zhang et al. \cite{zhang2025learning1} further enhanced DRL for visuomotor learning by introducing differentiable physics, enabling more efficient and accurate policy optimization through low-variance gradients. 

Distinct from previous efforts that directly optimize over raw pixels, our approach focuses on extracting compact and task-relevant representations from high-dimensional sequential visual inputs. This design aims to improve training efficiency and enable better transferability, especially under challenging real-world conditions.

\subsection{Representation Learning for RL}
Learning effective representations is crucial for visual RL tasks. Existing approaches can be broadly categorized into two types: auxiliary self-supervised learning \cite{yarats2021improving, yarats2021reinforcement, yoo2024mono} and latent world modeling for planning \cite{zhang2019solar, lee2020stochastic, zhao2023learning}. Our work complements the first category. 

Among auxiliary self-supervised methods, one strategy is pixel-level reconstruction \cite{yoo2024mono, bonatti2020learning, kargar2022increasing}, where the encoder is trained to reconstruct raw input observations, thereby encouraging the extraction of informative features that preserve task-relevant information. Another strategy is instance-level contrastive learning \cite{xing2024contrastive, zhang2025learning, de2022depth, fu2023learning, jiang2024bevnav}, which aims to produce similar representations for positive pairs and dissimilar ones for negatives, thereby enhancing discriminative capacity. For instance, Xing et al. \cite{xing2024contrastive} introduced an adaptive multi-pair contrastive learning method into DRL for visuomotor policy learning. Positive pairs were constructed from spatially adjacent frames along the flight track, while distant frames were treated as negatives, enabling local continuity in the learned representation. Zhang et al. \cite{zhang2025learning} further proposed a cross-modal contrastive learning approach to improve domain transfer of visuomotor policies. The contrastive objective was designed to extract depth-aware features from monocular images by aligning RGB–depth pairs, resulting in consistent navigation performance under visual interference and environmental changes. 

Different from prior works that utilize multi-layer perceptrons or DNNs to interpret pixels, our work extends instance-level contrastive learning by incorporating Transformer \cite{vaswani2017attention} as a powerful sequential modeling backbone to process consecutive frames. It allows the encoder to model temporal dependencies explicitly, which enhances context understanding and supports temporally consistent decision making.

\subsection{Sample-Efficient RL}
Improving sample efficiency remains a long-standing challenge in visual RL, particularly in complex real-world robotics tasks such as visuomotor policy learning. Low sample efficiency often leads to prohibitive training time, excessive trial-and-error interactions, and poor transferability \cite{zhang2021sample}. 

Existing methods can be broadly categorized into two directions: input-level optimization \cite{laskin2020curl, yarats2021image, laskin2020reinforcement} and policy-level guidance \cite{hester2018deep, wu2023human, zhang2024npe, huang2024safety}. The former direction focuses on reducing the burden of policy optimization by learning low-dimensional, compact visual representations from high-dimensional observations (e.g., CURL \cite{laskin2020curl}). These approaches align with representation learning techniques discussed in the previous section. Moreover, recent efforts have incorporated data augmentation \cite{yarats2021image, laskin2020reinforcement} into the RL pipeline, improving generalization from limited data and enhancing resilience to visual variations. The latter direction introduces prior demonstrations to assist policy learning. Notable progress has been made in methods like deep Q-learning from demonstrations (DQfD) \cite{hester2018deep}, where the replay buffer is enriched with prior demonstrations to accelerate early-stage learning. Beyond fixed demonstrations, recent studies \cite{wu2023human, zhang2024npe, huang2024safety} explore the use of adaptive guidance by dynamically updating prior policies during training. For instance, Wu et al. \cite{wu2023human} proposed a human-guided RL framework for autonomous navigation, where human feedback is used to interactively update a prior policy. This strategy improved sample efficiency by enabling online correction without extensive manual demonstrations. Zhang et al. \cite{zhang2024npe} introduced a non-expert policy-guided DRL framework, in which a suboptimal planner provided reference actions in the early training stage. The guidance was gradually decayed as train progressed to allow for independent exploration, resulting in enhanced data efficiency with reduced cost. In contrast to these methods, our work combines compact representation learning with privileged policy guidance in a unified framework, improving sample efficiency and real-world applicability under constrained perception.

\begin{figure*}[!t]
   \centering
   \includegraphics[width=1.0\textwidth]{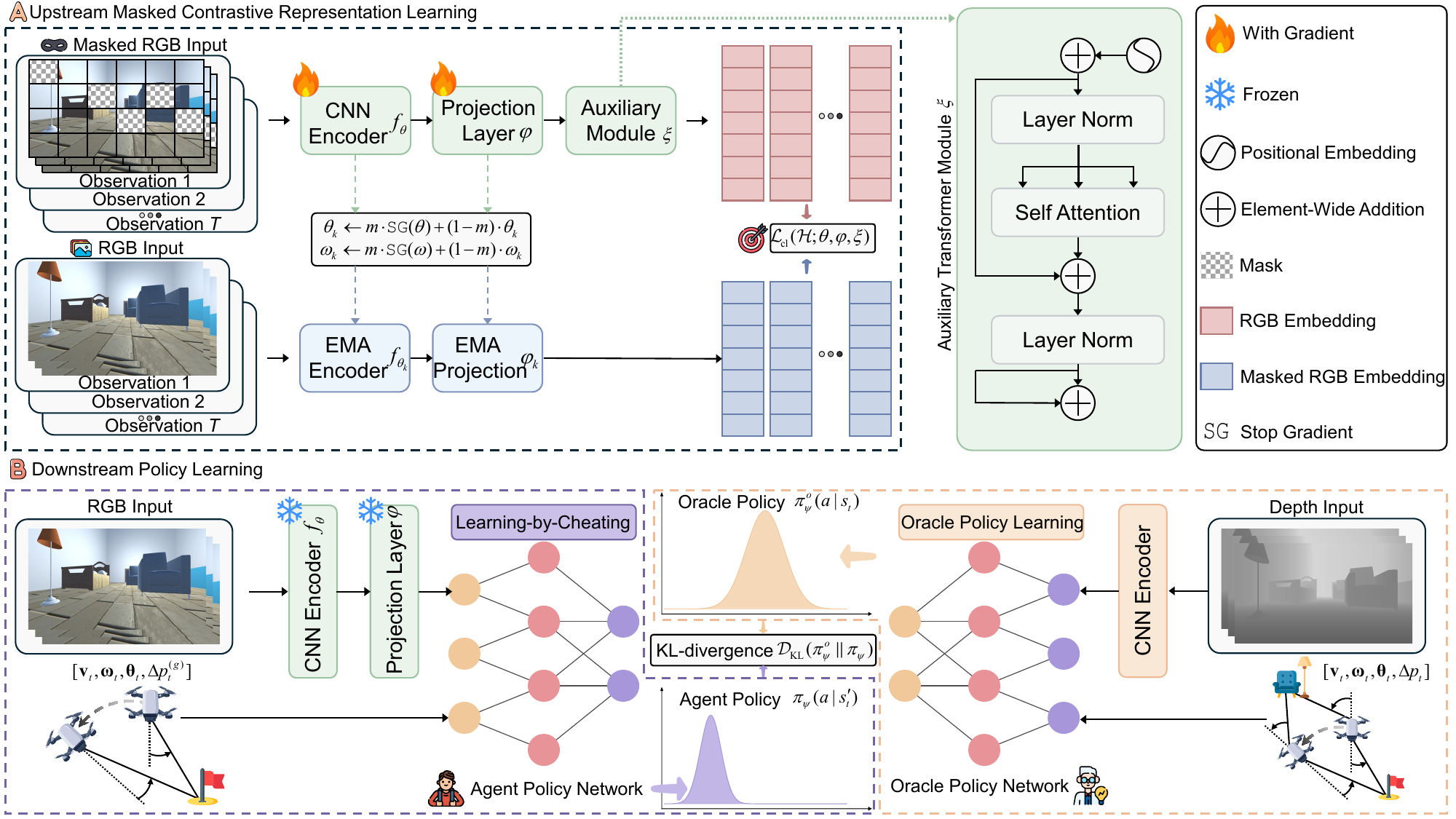} 
   \caption{The framework of OMC-RL. (a) Upstream Masked Contrastive Representation Learning: A masked contrastive learning module is used to learn compact and task-relevant visual representations from sequential RGB inputs. The masked and original inputs are processed through CNN encoders and projection layers, while the masked branch is further processed by an auxiliary transformer module to compute the contrastive loss $\mathcal{L}_\text{cl}$. After training, the CNN encoder is frozen and the transformer module is discarded. (b) Downstream Oracle-Guided Policy Learning: An oracle teacher policy is first trained using privileged depth inputs and full-state information, providing expert action distributions for downstream supervision. This oracle network supervises the student policy through a learning-by-cheating strategy. Specifically, the agent policy is optimized via KL-divergence against the oracle policy distribution to enable efficient visuomotor policy learning.}
   \label{fig:framework}
\end{figure*}

\section{Preliminaries and Problem Formulation}
This section presents the preliminaries of the proposed framework, including the problem formulation and the contrastive learning objective for visuomotor policy learning.

\subsection{Problem Formulation}
Given that visuomotor policy learning for robots is inherently subject to partial observability due to the limited field of view of the onboard monocular camera, we formulate the task as a partially observable Markov decision process (POMDP). The problem is defined by the tuple $(\mathcal{O}, \mathcal{A}, \mathcal{P}, \mathcal{R}, \gamma)$, where $\mathcal{O}$ denotes the observation space of high-dimensional image collections captured by the camera, and $\mathcal{A}$ is the action space available to the agent. At each timestep $t$, the agent observes an image $o_t \in \mathcal{O}$ and executes a corresponding action $a_t \in \mathcal{A}$, which jointly determine the evolution of its interaction with the environment. The transition dynamics model $\mathcal{P}$ specifies the conditional distribution over the next observation $o_{t+1}$ given the current input $o_t$ and action $a_t$, that is, $\mathcal{P} = \Pr(o_{t+1} \mid o_t, a_t)$. The reward function is defined as $\mathcal{R} : \mathcal{O} \times \mathcal{A} \mapsto \mathbb{R}$, which maps an observation-action pair $(o_t, a_t)$ to a scalar reward, \textit{i.e.}, $r_t = \mathcal{R}(o_t, a_t) \in \mathbb{R}$. Finally, $\gamma \in (0,1]$ denotes the discount factor that balances short-term and long-term rewards.

In DRL, DNNs are employed to process high-dimensional observations and learn sequential decision-making policies through trial-and-error interactions. Since the agent cannot directly access the underlying environment state $s_t$, it must instead utilize a history of recent visual observations to ensure informative action execution. Let $h_t = (o_{t-L+1}, o_{t-L+2}, \dots, o_t)$ denote this temporal history. To extract actionable representations from raw inputs, we adopt a two-stage architecture: a DNN encoder $f_\theta$, parameterized by $\theta$, first maps the observation history $h_t$ into a compact latent feature, which is then passed to a policy network $\pi_\psi$, parameterized by $\psi$, to generate the control action, \textit{i.e.}, $a_t = \pi_\psi(f_\theta(h_t))$. The learning objective is to obtain an optimal visuomotor policy $\pi_\psi^*: f_\theta(h_t) \mapsto \mathcal{A}$ that maximizes the expected cumulative discounted reward, $\mathbb{E} \left[ \sum_{t=0}^{\infty} \gamma^t r_t \right]$. During deployment, the agent follows the learned policy $\pi_\psi^* = \arg \max_{\pi_\psi} \mathbb{E}_{o \sim \mathcal{P},\, a \sim \pi_\psi} [ \mathcal{R}(o, a) ]$, where decisions are made iteratively by encoding the current observation history $h_t$ and sampling actions from the policy. 


\subsection{Contrastive Learning} Contrastive learning is a self-supervised representation learning paradigm that formulates instance-level discrimination as a classification task: each sample is treated as its own class and must be distinguished from all other instances in the dataset, without requiring any manual annotations. Given a batch of RGB images $\mathcal{I} = \{\boldsymbol{I}_1, \boldsymbol{I}_2, \dots, \boldsymbol{I}_N\}$ sampled from the replay buffer, each image $\boldsymbol{I}_i \in \mathcal{I}$, is processed by two DNN encoders: a query encoder $f_q$ and a key encoder $f_k$. These encoders map the inputs into a shared latent space, producing embeddings $q_i = f_q(\boldsymbol{I}_i)$ and $k_i = f_k(\boldsymbol{I}_i)$. The learning objective encourages the representation of each positive pair $(q_i, k_i)$ to be similar, while ensuring dissimilarity to the remaining keys in $\mathcal{K} \setminus \{k_i\}$. This is achieved using the InfoNCE loss \cite{laskin2020curl}:
\begin{equation}
\label{infonce}
\mathcal{L}_{\text{InfoNCE}} = - \mathbb{E}_{\mathcal{I}} \left[ \log \frac{\exp(\text{sim}(q_i, k_i)/\tau)}{\sum_{j=1}^{N} \exp(\text{sim}(q_i, k_j)/\tau)} \right],
\end{equation}
where $\text{sim}(\cdot, \cdot)$ denotes a similarity metric (e.g., cosine similarity), and $\tau$ is a temperature hyperparameter. To enhance temporal sensitivity, we adopt a sequence-based encoding scheme where $L$ visual observations are temporally stacked and encoded into a unified latent representation before applying the contrastive loss. This promotes the learning of temporally consistent features that are beneficial for downstream visuomotor policy learning.

\section{Methodology}
This section details the proposed framework, including masked contrastive learning, oracle policy learning, and the two-stage visuomotor policy training procedure.

\subsection{Overview}
We present a comprehensive framework for learning robust visuomotor policies under partial observability, formulated as a two-stage learning process. The full training pipeline is illustrated in Fig.~\ref{fig:framework}. Specifically, we decouple the visuomotor policy learning into two complementary components: (i) an upstream representation learning module based on masked contrastive learning, which extracts informative latent features from high-dimensional visual observations; and (ii) a downstream policy learning module based on a learning-by-cheating strategy, guided by privileged supervision from an oracle teacher policy. In the following subsections, the upstream representation learning is detailed in Section~\ref{text:masked}, and the downstream policy learning is presented in Section~\ref{text:policy}, including the oracle teacher policy and the learning-by-cheating strategy. The complete algorithm and the theoretical justification are summarized in Section~\ref{text:algorithm}.

\subsection{Upstream Masked Contrastive Representation Learning}
\label{text:masked}
To extract generalizable visual features, we decouple representation learning from policy optimization in RL \cite{stooke2021decoupling}, allowing the encoder to be trained via unsupervised objectives independently of reward signals. However, when decoupling is applied directly, RL algorithms typically sample uncorrelated mini-batches for policy optimization and discard collected trajectories after each update. This process disrupts the inherent temporal coherence among observations, potentially hindering stable learning. However, for visuomotor agents, sequential context is critical for informative decision-making due to the partially observable nature of image inputs. To address this, we store full trajectories and exploit the temporal continuity embedded in these sequences to enhance representation learning. Motivated by the success of Transformer architectures in modeling sequence data across domains such as translation \cite{devlin2019bert, cui2021pre}, we propose a masked contrastive learning framework based on the Transformer architecture to learn temporally consistent representations.

Specifically, at each iteration, a stack of $L$ consecutive images $h_t \in \mathcal{H}$, where $\mathcal{H}$ represents the set of all $h_t$, is first processed by the DNN encoder $f_\theta: \mathcal{H} \mapsto \mathbb{R}^d$ to obtain a compact $d$-dimensional latent representation $z_t = f_\theta(h_t)$. Unlike conventional approaches such as CURL \cite{laskin2020curl}, which directly calculates the contrastive loss on these latent features, we introduce an auxiliary transformation pipeline to further process the representation. Specifically, the representation $z_t$ is first passed through a non-linear projection module $\varphi$, and then encoded by a Transformer module $\xi$ to jointly model temporal dependencies and contextual structure. The final output from $\xi(\varphi(z_t))$ is used as the anchor embedding for contrastive learning. Subsequently, we will proceed to detail the auxiliary architecture, including data collection and model structure.

 \textbf{Data Collection.} We collect our training dataset by recording temporally ordered visual observations from drone flights under manual control, random exploration, and pretrained policies. These sequences are saved chronologically and used for offline training. At each training step, we sample $T$ sequential observations denoted by $\mathcal{H} = (h_1, h_2, \dots, h_T)$ from the dataset. To enable masked representation learning as described in \cite{devlin2019bert}, we define a binary masking vector $M = (M_1, M_2, \dots, M_T) \in \{0,1\}^T$, where each $M_i$ is independently sampled from a Bernoulli distribution with masking probability $\varrho_m \in [0,1]$. If $M_i = 1$, the corresponding input $h_i$ is altered via one of the following stochastic operations: with probability $80\%$, $h_i$ is replaced with a zero tensor; with probability $10\%$, it is substituted with a randomly sampled input $h_j$ from the dataset where $j \neq i$; and with probability $10\%$, it remains the same. We denote the processed observation as $h_i^{m}$, and obtain the final masked input via: $\bar{h}_i^{m} = M_i \cdot h_i^{m} + (1 - M_i) \cdot h_i$. The masked sequence is defined as $\mathcal{H}^{m} = (\bar{h}_1^{m}, \bar{h}_2^{m}, \dots, \bar{h}_T^{m})$ and passed through the encoder $f_\theta$ to generate latent embeddings $\hat{Z} = (z_1, z_2, \dots, z_T)$, where $z_i = f_\theta(\bar{h}_i^{m})$, $\bar{h}_i^{m} \in \mathcal{H}^{m}$.

\textbf{Model Structure.} We adopt a lightweight DNN encoder $f_\theta$ to extract compact features from consecutive observations. The encoder consists of two convolutional layers with 32 output channels: the first uses a $3 \times 3$ kernel with stride 2 for spatial downsampling, and the second applies a $3 \times 3$ kernel with stride 1 for feature refinement. Each convolutional layer is followed by a $\mathrm{ReLU}$  activation. The output is then flattened and passed through a fully connected layer, followed by layer normalization and a $\mathrm{tanh}$ activation, producing a bounded $d$-dimensional latent representation.

We introduce a projection module $\varphi$ that applies a nonlinear transformation to the latent features produced by the encoder before contrastive loss computation. This nonlinear mapping aims to project the features into a space where they are more linearly separable for the contrastive objective \cite{chen2020simple}. The module operates on the masked feature sequence $\hat{Z} = (z_1, z_2, \dots, z_T)$. Each $z_i$ is passed through a two-layer multilayer perceptron with $\mathrm{ReLU}$ activation. Specifically, we compute the projected representation as follows:

\begin{equation}
    \varphi(z_i) = W^{(2)} \cdot \text{ReLU}(W^{(1)} z_i + b^{(1)}) + b^{(2)}, 
\end{equation}
where $z_i \in \hat{Z}$, $W^{(1)}, W^{(2)}$ and $b^{(1)}, b^{(2)}$ are learnable weights and biases. The resulting transformed sequence $\hat{Z}^{(0)} = (\varphi(z_1), \dots, \varphi(z_T))$ is subsequently passed to the Transformer encoder for temporal modeling.

To reconstruct the masked latent representations rather than raw pixel observations, we employ a Transformer encoder module $\xi$ that refines the partially corrupted embeddings $\varphi(z_i) \in \hat{Z}^{(0)}$ into temporally contextualized representations. The architecture of $\xi$ follows the encoder-only architecture of the standard Transformer \cite{vaswani2017attention} with single-head attention and is composed of $L$ identical blocks. Each block contains a self-attention layer $\mathrm{Attn}(\cdot)$ and a feedforward transformation $\mathcal{F}(\cdot)$, equipped with residual connections and layer normalization.

Prior to entering the first block, positional encodings are added to each token to incorporate sequential structure, i.e., $\tilde{z}_i = \varphi(z_i) + p_i$, where $p_i$ denotes the sinusoidal positional embedding following the standard Transformer setting. At each block $l \in \{1, \dots, L\}$, the attention mechanism operates on the entire input sequence from the previous block to output $\tilde{Z}^{(l)} = (\tilde{z}_1^{(l)}, \dots, \tilde{z}_T^{(l)})$, where $\tilde{z}_i^{(l)}$ is the result of self-attention at position $i$ in layer $l$. Query, key, and value vectors are computed as $q_i^{(l)} = W_Q^{(l)} \tilde{z}_i^{(l-1)}$, $k_j^{(l)} = W_K^{(l)} \tilde{z}_j^{(l-1)}$, and $v_j^{(l)} = W_V^{(l)} \tilde{z}_j^{(l-1)}$, where $W_Q^{(l)}, W_K^{(l)}, W_V^{(l)}$ are learnable projection matrices for the $l$-th block. Here, $\tilde{z}_i^{(l)}$ can be computed as follows:
\begin{equation}
\begin{aligned}
\tilde{z}_i^{(l)} &=  \mathrm{Attn}(\tilde{z}_i^{(l-1)},\tilde{Z}^{(l-1)}) = \sum_{j=1}^{T} \alpha_{ij}^{(l)} v_j^{(l)},\\
\alpha_{ij}^{(l)} &= \frac{\exp \left( (q_i^{(l)})^\top k_j^{(l)} / \sqrt{d} \right)}{\sum_{j'=1}^{T} \exp \left( (q_i^{(l)})^\top k_{j'}^{(l)} / \sqrt{d} \right)}.
\end{aligned}
\label{eq:attention}
\end{equation}


Then, the output $\tilde{z}_i^{(l)}$ is passed through $\mathcal{F}(\cdot)$ consisting of a two-layer MLP with ReLU activation:
\begin{equation}
z_i^{(l)} = \mathcal{F}(\tilde{z}_i^{(l)}) = W_2^{(l)} \cdot \mathrm{ReLU}(W_1^{(l)} \tilde{z}_i^{(l)} + b_1^{(l)}) + b_2^{(l)}, 
\end{equation}
where all $W_1^{(l)}, W_2^{(l)}$ and $b_1^{(l)}, b_2^{(l)}$ are learnable parameters. By repeating the aforementioned operations for $L$ times, we obtain the final output sequence $Z^{(L)} = (z_1^{(L)}, z_2^{(L)}, \dots, z_T^{(L)})$, which serves as the input to the contrastive learning objective.

\textbf{Training Procedure.} Here we detail the optimization procedure of $f_\theta$, $\varphi$, and $\xi$. Specifically, we employ a query encoder $f_\theta$ and a key encoder $f_{\theta_k}$ with identical architecture but separate parameters, along with projection heads $\varphi$ and $\varphi_k$ of the same configuration. The query embeddings $q_i$ are obtained from the masked representations $Z^{(L)}$ produced by the Transformer module $\xi$, while the key set $\mathcal{K} = (k_1, k_2, \dots, k_T)$ is constructed from the non-masked inputs $\mathcal{H}$, where $k_i = \varphi_k(f_{\theta_k}(h_i))$, $h_i \in \mathcal{H}$. The key network parameters are updated using a momentum update rule \cite{laskin2020curl}: both the encoder parameters $\theta_k$ and the projector parameters $\omega_k$ of $\varphi_k$ are updated as $\theta_k \leftarrow m \cdot \texttt{SG}(\theta) + (1 - m) \cdot \theta_k$ and $\omega_k \leftarrow m \cdot \texttt{SG}(\omega) + (1 - m) \cdot \omega_k$, where $m \in [0,1]$ is the momentum coefficient and $\texttt{SG}(\cdot)$ denotes the stop-gradient operation \cite{he2020momentum}.

The contrastive loss is computed over the masked positions, encouraging each reconstructed representation $q_i$ (i.e., $z_i^{(L)} \in Z^{(L)}$) to be similar to its corresponding positive key $k_i$ while dissimilar to the rest of the batch. This formulation aims to enforce the encoder $f_\theta$ to preserve the most representative and semantically relevant features in the original pixel space. Only by retaining globally consistent patterns can the model accurately reconstruct missing information from partially observed inputs. Note that the loss in OMC-RL follows the same InfoNCE formulation as in Eq. (\ref{infonce}), with the only difference that $q_i$ and $k_i$ are derived from masked reconstruction. Formally, the loss is defined as:
\begin{equation}
\label{cl loss}
\mathcal{L}_{\text{cl}}(\mathcal{H};\theta, \varphi, \xi) = - \mathbb{E}_{q_i} \left[ M_i \cdot \log \frac{\exp\left( \mathrm{sim}(q_i, k_i)/\tau \right)}{\sum_{j=1}^{T} \exp\left( \mathrm{sim}(q_i, k_j)/\tau \right)} \right].
\end{equation}

After training, we discard $\xi$ and retain only $f_\theta$ and $\varphi$, whose parameters are frozen and used as the visual backbone for downstream RL policy optimization, thereby reducing computational overhead.

\subsection{Downstream Oracle-Guided Policy Learning}
\label{text:policy}
\textbf{Oracle Policy Learning.} To provide additional guidance for training the visuomotor policy, we first pretrain an oracle teacher policy, denoted as $\pi_{\psi}^{o}$. The oracle teacher $\pi_{\psi}^{o}$ is a state-based policy with access to the full global state information $s_t$, which is defined as follows:

\begin{equation}
    s_t = \left(\mathcal{I}_t^{\mathrm{d}}, \mathbf{v}_t, \boldsymbol{\omega}_t, \boldsymbol{\theta}_t, \Delta p_t\right), 
\end{equation}
where $\mathcal{I}_t^{\mathrm{d}} = \{\boldsymbol{I}_{t-L+1}^d, \boldsymbol{I}_{t-L+2}^d, \dots, \boldsymbol{I}_{t}^d\}$ denotes a sequence of depth images with $\boldsymbol{I}_{t}^d \in \mathbb{R}^{H \times W}$, $\mathbf{v}_t \in \mathbb{R}^3$ and $\boldsymbol{\omega}_t \in \mathbb{R}^3$ are drone's linear and angular velocities, respectively, $\boldsymbol{\theta}_t \in \mathbb{R}^3$ represents the drone's orientation, and $\Delta p_t \in \mathbb{R}^6$ represents the relative positions between the drone and surrounding objects, including obstacles and the goal. The oracle teacher $\pi_{\psi}^{o}$ maps this global state $s_t$ to a control command $\mathbf{u}_t$ as: $\pi_{\psi}^{o}: s_t \mapsto \mathbf{u}_t \in \mathbb{R}^3$, where $\mathbf{u}_t = (v_x, v_y, \omega_z)$ includes the linear velocities in the $x$ and $y$ directions and the angular velocity around the $z$-axis.

The network architecture of $\pi_{\psi}^{o}$ consists of a visual encoder $f_\theta^o$ followed by a fully-connected layer $\varphi^o$. The encoder $f_{\theta}^o: \mathcal{I}_t^{\mathrm{d}} \mapsto \mathbb{R}^d$ maps the sequential depth images $\mathcal{I}_t^{\mathrm{d}}$ into a $d$-dimensional feature representation $z_t^{\mathrm{d}} = f_\theta^o(\mathcal{I}_t^{\mathrm{d}})$. This embedding is then concatenated with the remaining state information and passed into the fully connected module  $\varphi^o$ to produce the control output: $\mathbf{u}_t^o =  \varphi^o\left( \left[z_t^{\mathrm{d}} ; \mathbf{v}_t ; \boldsymbol{\omega}_t ; \Delta p_t \right] \right)$.  The oracle teacher $\pi_{\psi}^{o}$ can be trained using various model-free RL algorithms \cite{mnih2015human, haarnoja2018soft, schulman2017proximal}. In this work, we adopt Proximal Policy Optimization (PPO) \cite{schulman2017proximal} for its stability and effectiveness in continuous control tasks, which are critical properties for agile and precise visuomotor policy learning. After training, it serves as guidance for subsequent visuomotor policy learning via a learning-by-cheating strategy.


\textbf{Learning-by-Cheating Strategy.} After upstream representation learning, the frozen $f_\theta$ and $\varphi$ are used to process sequential RGB inputs for downstream RL policy learning. Given the current observation history $h_t$ composed of $L$ stacked RGB frames, the $d$-dimensional latent visual representation is computed as: $z_t = \varphi(f_\theta(h_t)),\ z_t \in \mathbb{R}^d$. Unlike the oracle teacher policy $\pi_{\psi}^{o}$, which operates on the privileged state $s_t = (\mathcal{I}_t^{\mathrm{d}}, \mathbf{v}_t, \boldsymbol{\omega}_t, \boldsymbol{\theta}_t, \Delta p_t)$, the RL policy $\pi_{\psi}$ is trained using a partial observation $o_t = (z_t, \mathbf{v}_t, \boldsymbol{\omega}_t, \boldsymbol{\theta}_t, \Delta p_t^{(g)})$, where $\Delta p_t^{(g)} \in \mathbb{R}^3$ represents the relative position between the agent and the goal. The control output is then predicted by passing $o_t$ through the actor network $\phi_a$: $\mathbf{u}_t = \phi_a(o_t)$.

We adopt PPO \cite{schulman2017proximal} as the RL algorithm to optimize $\pi_{\psi}$. Given a trajectory $\mathcal{T}$ sampled from rollouts, the loss is defined as:
\begin{equation}
\label{ppo loss}
\begin{aligned}
\mathcal{L}_{\text{rl}}(\mathcal{T}; \phi_a, \phi_c) = & -\mathbb{E}_t \left[ 
\min \left( r_t \hat{A}_t,\ \text{clip}(r_t, 1 - \epsilon, 1 + \epsilon) \hat{A}_t \right) 
\right] \\
& + \mathbb{E}_t \left[ \left( V(o_t) - \hat{V}_t \right)^2 \right],
\end{aligned}
\end{equation}
where $\phi_a$ and $\phi_c$ are the parameters of the actor and critic networks, respectively, $t$ indexes timesteps sampled from $\mathcal{T}$, $r_t$ denotes the probability ratio between the new and old policies, $\hat{A}_t$ is the estimated advantage, $V(o_t)$ is the predicted state value, and $\hat{V}_t$ is the Monte Carlo return computed from observed rewards.

To leverage the oracle teacher policy $\pi_{\psi}^{o}$ for guiding the training, we adopt the learning-by-cheating paradigm, in which the agent is supervised by an oracle policy that has privileged access to full-state information unavailable during deployment. This allows for stronger and more informative guidance throughout training. Specifically, we introduce an additional KL-divergence term between the action distributions of $\pi_{\psi}$ and $\pi_{\psi}^{o}$. Let $\pi_{\psi}(\cdot|o_t)$ and $\pi_{\psi}^{o}(\cdot|s_t)$ denote the stochastic policies, the KL divergence is defined as:
\begin{equation}
\label{KL}
\mathcal{D}_{\mathrm{KL}}(\pi_{\psi}^{o} \parallel \pi_{\psi}) = \mathbb{E}_{a \sim \pi_{\psi}^{o}(\cdot|s_t)} \left[ 
\log \frac{\pi_{\psi}^{o}(a|s_t)}{\pi_{\psi}(a|o_t)} 
\right],
\end{equation}
which penalizes deviations from the oracle under the privileged state distribution. To ensure stability, this term is scaled by a decaying coefficient $\alpha$ and a fixed scaling constant $\beta$ following \cite{zhang2024npe}, leading to the final RL policy objective:
\begin{equation}
\label{total loss}
\mathcal{L}_{\pi_{\psi}}(\mathcal{T}; \phi_a, \phi_c) = (1 - \alpha)  \mathcal{L}_{\text{rl}}(\mathcal{T}; \phi_a, \phi_c) + \alpha \beta  \mathcal{D}_{\mathrm{KL}}(\pi_{\psi}^{o} \parallel \pi_{\psi}).
\end{equation}

The coefficient $\alpha$ decays over time to allow the learned policy to gradually diverge from the oracle and explore its own strategies. Finally, the overall learning objective of our framework is formulated as:
\begin{equation}
\min_{\theta, \varphi, \xi, \phi_a, \phi_c} \mathcal{L}_{\text{cl}}(\mathcal{H}; \theta, \varphi, \xi) + \mathcal{L}_{\pi_{\psi}}(\mathcal{T}; \phi_a, \phi_c).
\label{eq:overall_obj}
\end{equation}

\subsection{Summary and Theoretical Justification}
\label{text:algorithm}
\textbf{Algorithm Summary.} We summarize the overall training pipeline of OMC-RL in Algorithm \ref{alg:training}, which integrates upstream masked contrastive representation learning and downstream policy learning under oracle guidance. In the first stage, the agent learns robust visual representations through masked contrastive learning. In the second stage, a visuomotor policy is optimized using a learning-by-cheating strategy, with supervision provided by an oracle teacher policy that has access to privileged information.

\begin{algorithm}[t]
\caption{Training Pipeline of OMC-RL}
\label{alg:training}
\KwIn{RGB sequence dataset $\mathcal{H}$, rollout trajectory $\mathcal{T}$, and oracle teacher $\pi_{\psi}^{o}$}
\KwOut{Optimized visuomotor policy $\pi_{\psi}$}

\BlankLine
\textbf{// Upstream Training} \\

\ForEach{iteration}{
    Sample sequence $\mathcal{H} = (h_1, \dots, h_T)$ \\
    Sample mask $M \sim \text{Bernoulli}(\varrho_m)$ \\
    Construct $\mathcal{H}^m = (\bar{h}_1^m, \dots, \bar{h}_T^m)$ \\
    Encode: $z_i = f_\theta(\bar{h}_i^m)$,\quad $q_i = \xi(\varphi(z_i))$ \\
    Compute keys: $k_j = \varphi_k(f_{\theta_k}(h_j)),\ h_j \in \mathcal{H}$ \\
    Compute $\mathcal{L}_{\text{cl}}(\mathcal{H}; \theta, \varphi, \xi)$ via Eq.~\eqref{cl loss} \\
    
    $\theta \leftarrow \theta - \lambda_\theta \nabla_\theta \mathcal{L}_{\text{cl}}(\mathcal{H}; \theta, \varphi, \xi)$ \\
    $\varphi \leftarrow \varphi - \lambda_\varphi \nabla_\varphi \mathcal{L}_{\text{cl}}(\mathcal{H}; \theta, \varphi, \xi)$ \\
    $\xi \leftarrow \xi - \lambda_\xi \nabla_\xi \mathcal{L}_{\text{cl}}(\mathcal{H}; \theta, \varphi, \xi)$ \\
    
    $\theta_k \leftarrow m \cdot \texttt{SG}(\theta) + (1 - m) \cdot \theta_k$ \\
    $\omega_k \leftarrow m \cdot \texttt{SG}(\omega) + (1 - m) \cdot \omega_k$ \\
}

\BlankLine
\textbf{// Downstream Training} \\

Freeze $f_\theta$ and $\varphi$, discard $\xi$ \\

\ForEach{iteration}{
    Sample rollout trajectory $\mathcal{T} = \{s_t, o_t, \mathbf{u}_t, \mathbf{u}_t^o, r_t\}_{t=1}^T$ \\
    Compute $\mathcal{L}_{\text{rl}}(\mathcal{T}; \phi_a, \phi_c)$ via Eq.~\eqref{ppo loss} \\
    Compute $\mathcal{D}_{\mathrm{KL}}(\pi_{\psi}^{o} \parallel \pi_{\psi})$ via Eq.~\eqref{KL} \\
    Compute $\mathcal{L}_{\pi_\psi}(\mathcal{T}; \phi_a, \phi_c)$ via Eq.~\eqref{total loss} \\
    
    $\phi_a \leftarrow \phi_a - \lambda_{\phi_a} \nabla_{\phi_a} \mathcal{L}_{\pi_\psi}(\mathcal{T}; \phi_a, \phi_c)$ \\
    $\phi_c \leftarrow \phi_c - \lambda_{\phi_c} \nabla_{\phi_c} \mathcal{L}_{\pi_\psi}(\mathcal{T}; \phi_a, \phi_c)$ \\
    
}
\end{algorithm}

\textbf{Theoretical Justification.} The objective of masked contrastive learning is to discover a representation space $Z$ that captures the structural granularity and alignment among samples. Assuming the existence of an optimal representation $Z^\star$, we can define the expected population risk under $Z^\star$ as:

\begin{equation}
R_E(\theta, \varphi, \xi) = \mathbb{E}_{h_i}[\mathcal{L}_{\text{cl}}(h_i, z_i^\star; \theta, \varphi, \xi)].
\end{equation}

In practice, the corresponding empirical risk over $N$ samples is given by:

\begin{equation}
\hat{R}(\theta, \varphi, \xi) = \frac{1}{N} \sum_{i=1}^N \mathcal{L}_{\text{cl}}(h_i, z_i^\star; \theta, \varphi, \xi).
\end{equation}

\begin{theorem}\label{theorem1}
The generalization error of masked contrastive learning is bounded by the discrepancy between the learned representation space $Z$ and the optimal representation space $Z^\star$, i.e.,
\begin{equation}
    d_z = \frac{1}{n} \sum_{i=1}^n \lVert z_i - z_i^\star \rVert_2.
\end{equation}
\end{theorem}

\begin{proof}
Assume the contrastive loss $\mathcal{L}_{\text{cl}}(h_i, z_i^\star; \theta, \varphi, \xi)$ is bounded in $[a, b]$ and is $\lambda$-Lipschitz continuous with respect to $z_i$.  
Let $\mathcal{N}_\mathcal{E}$, $\mathcal{N}_\mathcal{P}$, and $\mathcal{N}_\mathcal{T}$ denote the covering numbers of the encoder, projection head, and transformer module spaces, respectively. Then, with probability at least $1 - \delta$, we have:
\begin{align*}
|R_E(\theta, \varphi, \xi) - R(\theta, \varphi, \xi)| 
&\leq \underbrace {|R_E(\theta, \varphi, \xi) - \hat{R}(\theta, \varphi, \xi)|}_{\text{Hoeffding's inequality}} \\
&\quad + \underbrace {|\hat{R}(\theta, \varphi, \xi) - R(\theta, \varphi, \xi)|}_{\text{Lipschitz continuous}} \\
&\leq |a - b| \sqrt{\frac{\log(2\mathcal{N}_\mathcal{E} \cdot \mathcal{N}_\mathcal{P} \cdot \mathcal{N}_\mathcal{T} / \delta)}{2n}} \\
&\quad + \lambda \cdot \frac{1}{n} \sum_{i=1}^n \lVert z_i - z_i^\star \rVert_2.
\end{align*}
\end{proof}

Based on Theorem \ref{theorem1}, the generalization error of masked contrastive learning can be effectively reduced when the learned representation space $Z$ better approximates the optimal space $Z^\star$. By employing a masking strategy, OMC-RL encourages the model to capture the intrinsic structure of the data, thus providing a theoretical justification for its effectiveness.

\section{Experimental Setup}
This section details the experimental setup for evaluating OMC-RL. We describe the simulation environments, evaluation metrics, baseline methods, and key implementation configurations used throughout our experiments.

\begin{figure}[t!]
    \centering
    \begin{subfigure}[htbp]{0.48\linewidth}
        \centering
        \includegraphics[width=\linewidth]{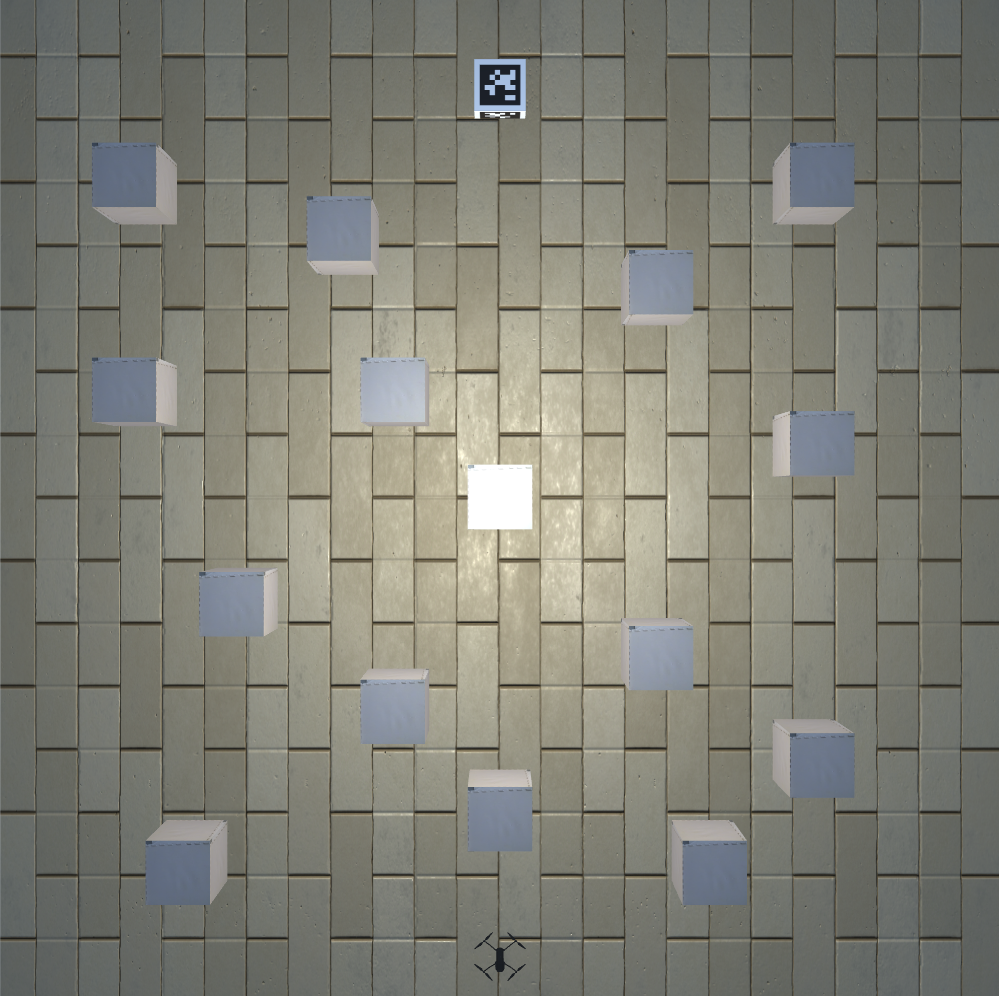}
        \caption{Box environment}
        \label{fig:sub1}
    \end{subfigure}
    \hfill
    \begin{subfigure}[htbp]{0.48\linewidth}
        \centering
        \includegraphics[width=\linewidth]{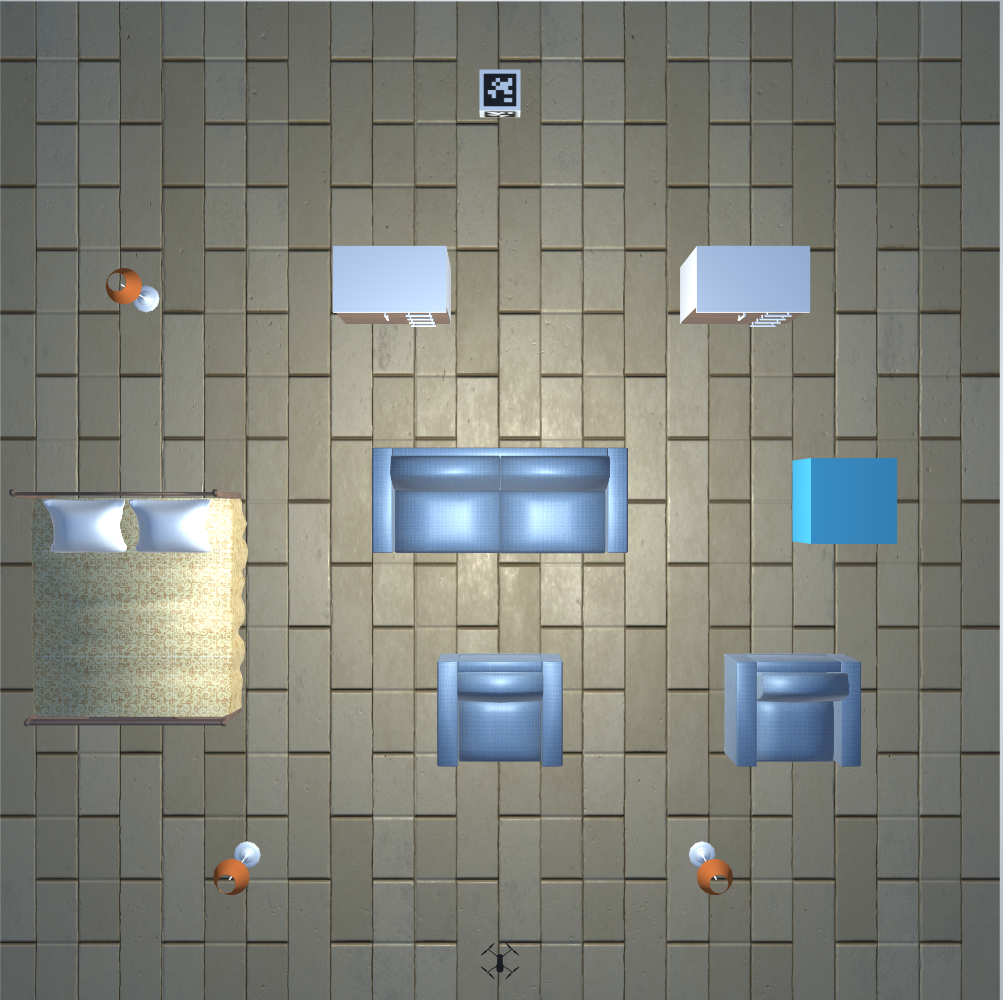}
        \caption{Furniture environment}
        \label{fig:sub2}
    \end{subfigure}
    
    \vskip\baselineskip
    
    \begin{subfigure}[htbp]{0.48\linewidth}
        \centering
        \includegraphics[width=\linewidth]{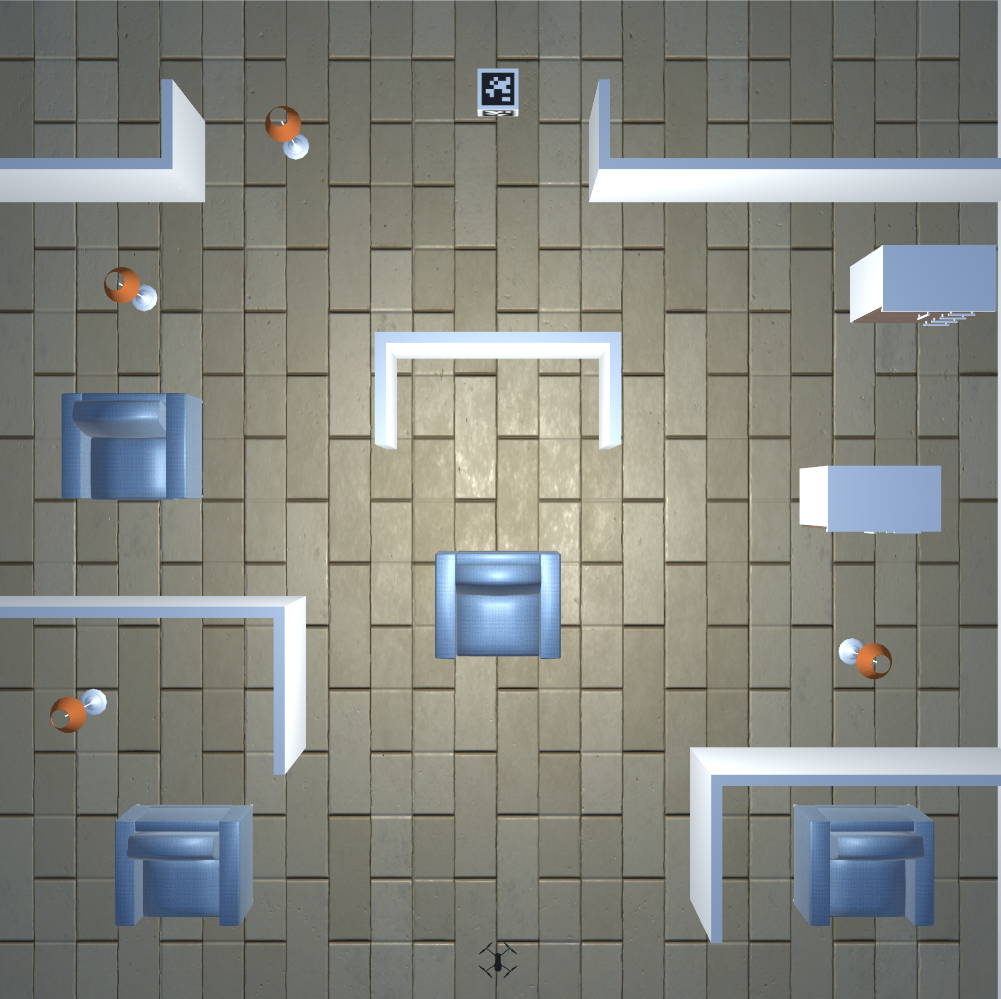}
        \caption{Barrier environment}
        \label{fig:sub3}
    \end{subfigure}
    \hfill
    \begin{subfigure}[htbp]{0.48\linewidth}
        \centering
        \includegraphics[width=\linewidth]{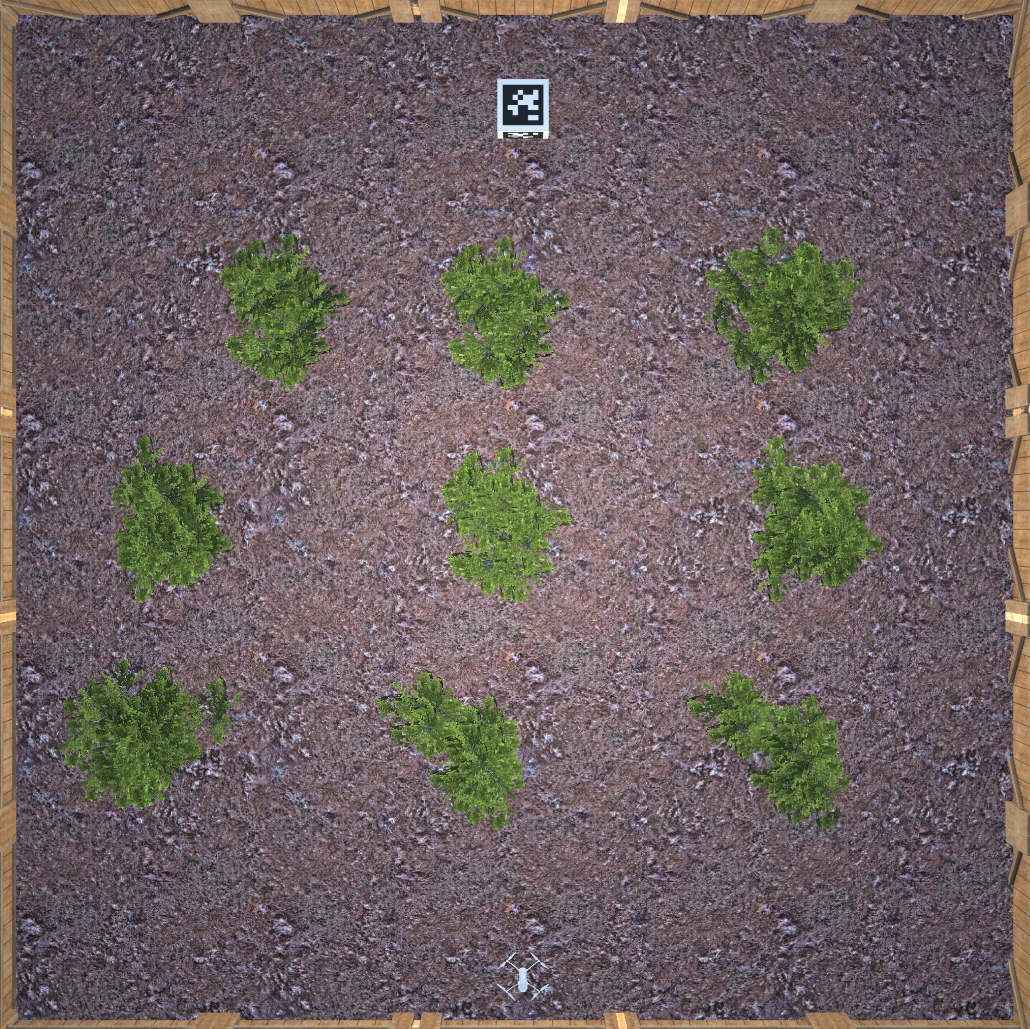}
        \caption{Tree environment}
        \label{fig:sub4}
    \end{subfigure}

    \caption{Simulation environments with increasing complexity used to evaluate OMC-RL.}
    \label{fig:scene}
\end{figure}

\subsection{Simulation Setup}
We construct our simulation environments in Unity to systematically evaluate the performance and generalization capability of our proposed method under diverse conditions. To this end, we design four distinct scenarios with various complexities, as illustrated in Fig. \ref{fig:scene}. The first is a structured indoor environment populated with identical box-shaped obstacles. The second is a furniture-filled indoor scene containing a variety of objects with different shapes, sizes, and materials. The third scenario introduces additional structural complexity by incorporating irregular barriers. The fourth setting is an outdoor forest environment densely populated with trees of varying geometries and textures, presenting significant visual diversity and occlusion. These environments provide a gradient of perceptual and navigational difficulty. For a fair comparison of sample efficiency, asymptotic performance, and generalization, all methods, including OMC-RL and baselines, are trained exclusively in the box-based environment and are evaluated in the remaining scenes without fine-tuning.

\subsection{Metrics} 
We evaluate the performance of our method using several widely adopted metrics \cite{anderson2018vision}. 
\begin{itemize} 
    \item \textit{Navigation Error (NE)}, defined as the average Euclidean distance between the drone's terminal position and the goal (m). 
    \item \textit{Oracle Success Rate (OS)}, which measures the percentage of episodes where any location along the trajectory falls within a fixed success threshold $\varepsilon$ of $0.5\,\mathrm{m}$ from the goal, as specified in our setup (\%). 
    \item \textit{Success Rate (SR)}, defined as the percentage of episodes in which the drone reaches the goal (\%). 
    \item \textit{Success weighted by Path Length (SPL)}, the primary metric reflecting both task completion and navigation efficiency, calculated as $\text{SPL} = \frac{1}{M} \sum_{i=1}^{M} \frac{\mathbb{I}_i \cdot \ell_i}{\max(d_i, \ell_i)}$, where $\mathbb{I}_i \in \{0,1\}$ indicates success, $\ell_i$ is the optimal path length, and $d_i$ is the executed trajectory length. 
    \item \textit{Collision Rate (CR)}, defined as the percentage of episodes in which the drone collides with obstacles (\%). 
    \item \textit{Time to Success (TTS)}, defined as the average number of steps taken in successful episodes. 
\end{itemize}

Note that these metrics are applied exclusively in simulation, where an episode is considered successful only if the drone reaches the goal precisely. For real-world experiments, only the success rate is reported, and an episode is deemed successful if the drone arrives within a $0.5\,\mathrm{m}$ radius of the goal to ensure safety.

\subsection{Baselines}
\label{text::baselines}
To comprehensively validate the performance of OMC-RL, we compare it with several representative baselines across RL and traditional planning-based approaches.

\begin{itemize}

    \item NPE \cite{zhang2024npe}: An improved RL baseline that combines non-optimal demonstrations IL to address the trade-off between sample efficiency and asymptotic performance. We use it to benchmark the effectiveness of our oracle-guided IL formulation.
    
    \item CURL \cite{laskin2020curl}: A classical contrastive RL framework that learns compact visual representations by maximizing similarity between augmented views of the same observation. It relies mainly on local image augmentations and CNN encoders without modeling temporal relationships.

    \item PPO \cite{schulman2017proximal}: A widely used model-free RL baseline. We adopt its end-to-end variant that directly maps monocular RGB inputs to continuous action commands without any auxiliary supervision or representation pretraining.


    \item Hybrid APF \cite{pan2021improved}: A traditional path planning baseline based on the artificial potential field (APF), augmented with A* search to improve global feasibility. This method represents classical non-learning approaches.
\end{itemize}

\subsection{Parameter and Architecture Configuration}

\textbf{Upstream Masked Contrastive Representation Learning.} The encoder is trained using a batch size of $32$ over sequences of length $T=16$, where each observation $h_t$ is formed by stacking $L=3$ consecutive RGB frames. Each frame is randomly cropped from $224 \times 224$ to $192 \times 192$ before being fed into the encoder. The output latent representation dimension is $d = 384$. We use a masking probability $\varrho_m = 0.5$, and set the temperature coefficient in the contrastive loss of Eq.~\eqref{cl loss} to $\tau = 0.07$. The encoder $f_\theta$ and projection head $\varphi$ are optimized using Adam with learning rates $\lambda_\theta = \lambda_\varphi = 1 \times 10^{-3}$, while the Transformer module $\xi$ is trained with a separate learning rate $\lambda_\xi = 2 \times 10^{-3}$. A warm-up and inverse square root decay schedule is applied to the Transformer for training stabilization, with a warm-up horizon of 6000 steps. The architecture details of the encoder, projection head, and Transformer are provided in Section~\ref{text:masked}. The momentum coefficient used to update the key encoder and key projection head is set to $m = 0.05$.

\textbf{Downstream Oracle-Guided Policy Learning.} The RL policy is optimized with a batch size of $1024$, buffer size of $10240$, and a linearly decaying learning rate initialized at $3 \times 10^{-4}$. The policy network consists of two fully connected layers with $256$ hidden units each, and layer normalization is applied to all input features. The clipping parameter is set to $\epsilon = 0.2$, and other hyperparameters follow the standard PPO \cite{schulman2017proximal} style, including a GAE parameter of $0.95$, time horizon of $128$, and three optimization epochs per iteration, etc. The oracle-guided KL regularization coefficient $\alpha$ is linearly annealed from $0.95$ to $0$ every $10000$ steps, with the scaling factor $\beta$ fixed at $1.0$.

To encourage goal reaching, safety, and efficiency, the reward function consists of three components. A terminal reward of $+10$ is given upon reaching the goal, while collisions incur a penalty of $-1$. To discourage overly long trajectories, each action step receives a penalty of $-\frac{1}{H_{\max}}$, where $H_{\max} = 5000$ is the maximum episode length. Additionally, a reward of $0.1 \times (d_{\text{init}} - d_t)$ is applied at each timestep, where $d_{\text{init}}$ and $d_t$ represent the initial and current distances to the goal, respectively. Note that all baselines adopt the same DNN architecture as $f_\theta$ to ensure fair visual representation capacity. For baselines involving overlapping training settings, such as the feature encoder used in CURL \cite{laskin2020curl}, and the RL optimization parameters in NPE \cite{zhang2024npe}, we apply identical hyperparameter settings to ensure fair comparison.

\section{Results and Analysis}
This section presents the experimental results and key analyses that validate the effectiveness of OMC-RL. We first evaluate its performance across a series of simulated environments and compare it against various baselines. We then assess its real-world generalizability through physical deployment on a quadrotor platform. Finally, we conduct ablation studies to examine the contribution of each core component.

\subsection{Comparison Results in Simulation}

All learning-based baselines are trained exclusively in the box-based environment. Their reward curves during training are illustrated in Fig. \ref{fig:reward}. After training, each baseline is evaluated in the remaining three environments without any fine-tuning. In each environment, we conduct 200 testing episodes to assess generalization. For the non-learning Hybrid APF, the policy is directly applied to all evaluation scenarios. Note that we refer to OMC-RL as Ours in all results.

\textbf{(1) Training Results.} As shown in Fig. \ref{fig:reward}, our method achieves the fastest convergence and strongest asymptotic performance, consistently outperforming all baselines. This advantage can be attributed to two key factors: (i) the visual encoder is pretrained and frozen, which reduces training instability and computational cost; and (ii) the introduction of oracle guidance during early training is conducive to acquiring informative priors, substantially accelerating learning. Among the remaining baselines, NPE benefits from suboptimal demonstrations that enhance initial exploration, but its reliance on imperfect guidance ultimately limits overall performance. CURL converges more slowly and achieves weaker asymptotic performance compared to NPE. It also underperforms our method, which can be attributed to its inability to model temporal relationships across observations. PPO, as the vanilla baseline, performs significantly inferior than all baselines.

\begin{figure}[t!]
    \centering
    \includegraphics[width=1.0\linewidth]{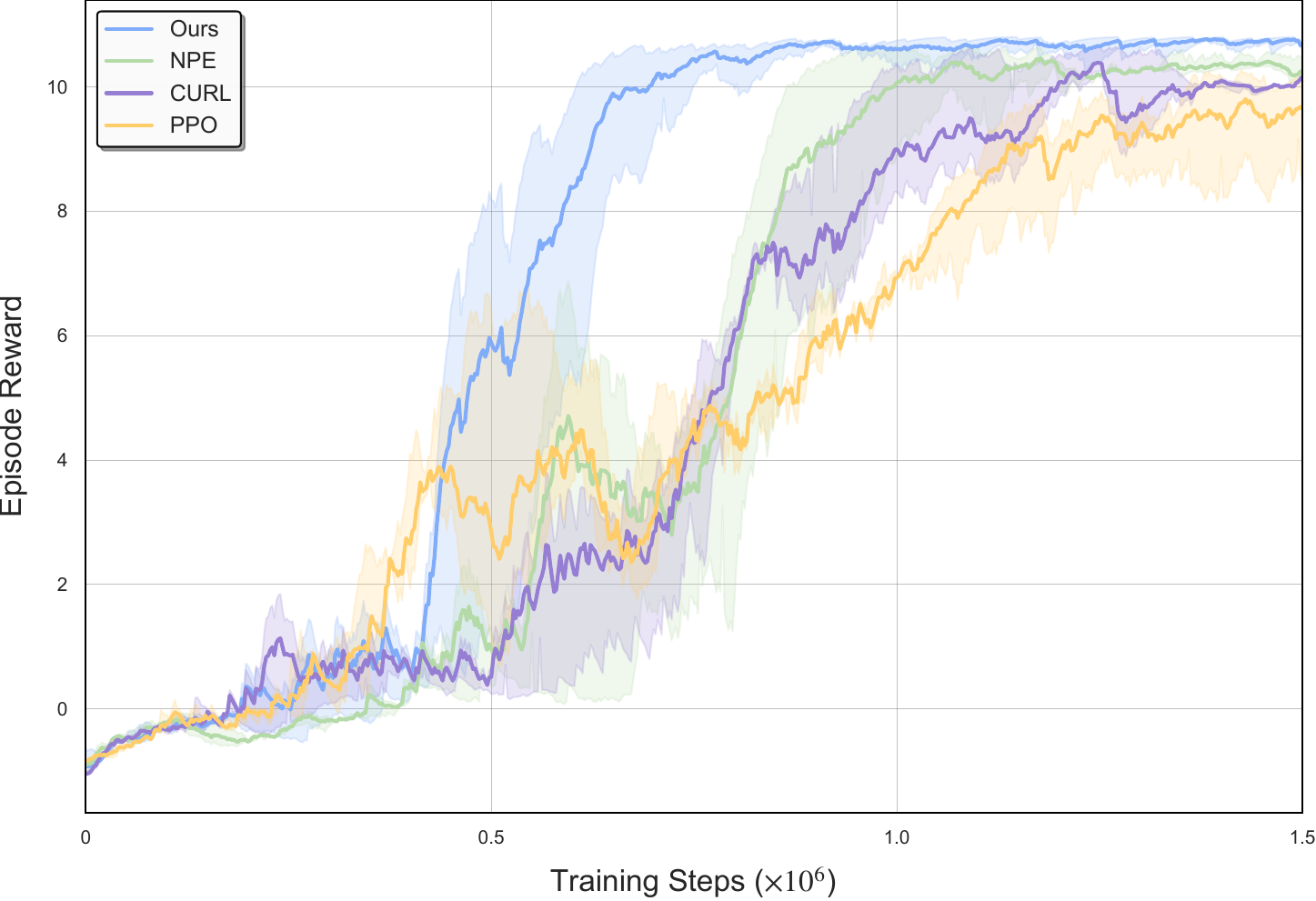}
    \caption{Training curves of episodic reward for all learning-based baselines. All results are averaged over three random seeds, with shaded regions indicating confidence intervals. Oracle serves as an upper bound with the fastest convergence and the highest asymptotic performance, while our method achieves comparable results and consistently outperforms all other baselines. NPE benefits from imitation of suboptimal demonstrations, outperforming CURL in both sample efficiency and asymptotic performance. PPO, as a vanilla baseline, exhibits the weakest performance.
}
    \label{fig:reward}
\end{figure}

\textbf{(2) Evaluation Results.} Quantitative results across all metrics are summarized in Table \ref{tab:baseline}. The results show that our method maintains consistently strong generalization across all evaluation environments despite relying solely on sequences of monocular RGB frames. This generalization capability stems from two key design choices: (i) the pretrained and frozen visual encoder facilitates feature transfer by learning generalizable visual representations; and (ii) our method explicitly models temporal relationships across sequential observations, which are critical for visuomotor policy learning under partial observability. In contrast, although CURL also follows a pretrain-and-freeze paradigm, its focus on frame-wise representation learning without temporal modeling limits its generalization, particularly in the outdoor environment with high visual variability. NPE exhibits noticeable performance drops in evaluation scenarios. This is expected, as IL approaches inherently rely on prior demonstrations, making them prone to overfitting and less adaptable to unseen settings. While NPE partially alleviates this by transitioning from IL to exploration, its end-to-end training setup with RGB observations, where the encoder and policy are jointly optimized, still leads to overfitting within the training domain. PPO, as the vanilla RL baseline, performs the worst among all learning-based methods. The Hybrid APF baseline demonstrates competitive results in structured environments such as the furniture and tree scenes. However, its reliance on hand-crafted potential fields makes it susceptible to local minima, especially in irregular barrier layouts. In such settings, potential field approximations become unreliable, often resulting in unstable trajectories. These results demonstrate that masked contrastive learning and oracle guidance jointly enhance both navigation performance and generalization to unseen environments.

Representative flight trajectories are visualized in Fig. \ref{fig:tra}, further demonstrating the effectiveness of our approach. Our method consistently outperforms all baselines across evaluation scenarios, generating smooth and coherent flight trajectories. This demonstrates the advantage of incorporating oracle guidance, as leveraging complete environmental state information enables the policy to perform more effective and goal-directed maneuvers.

\begin{figure*}[t!]
    \centering

    \begin{subfigure}{0.32\linewidth}
        \centering
        \includegraphics[width=\linewidth]{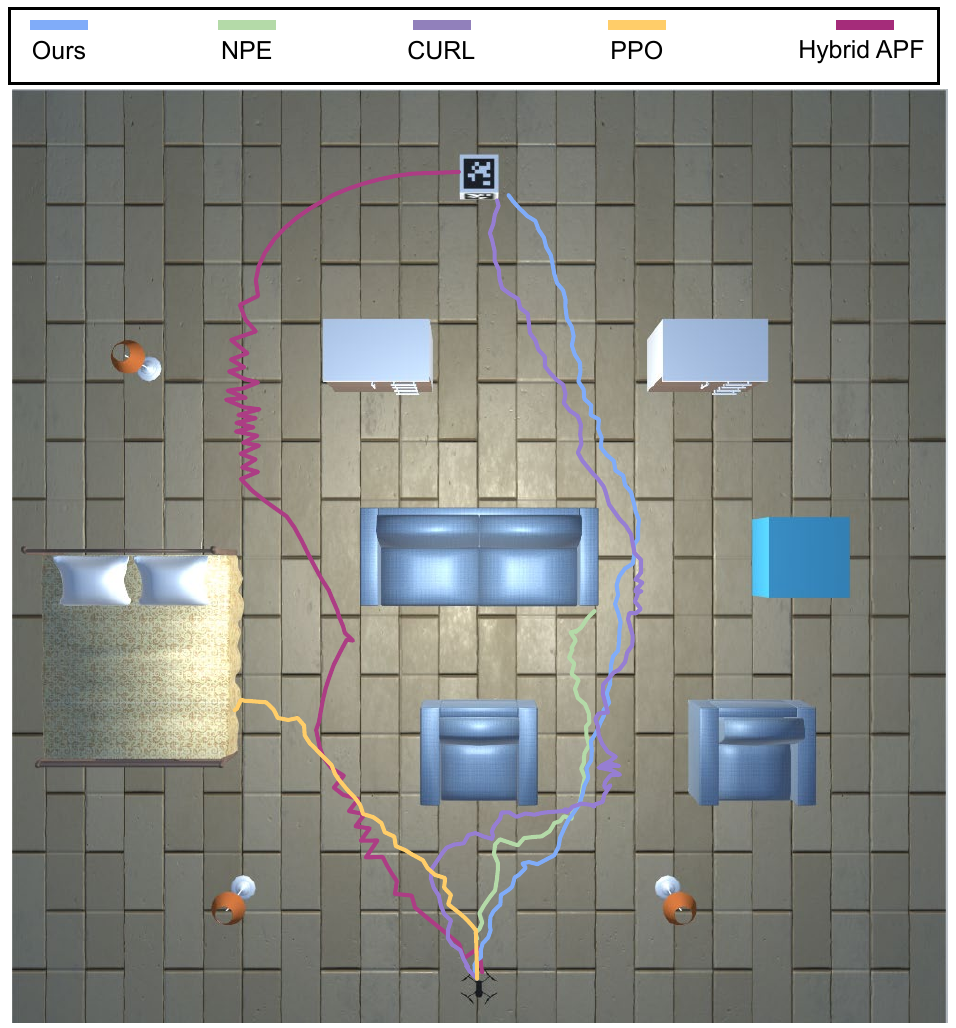}
        \caption{Furniture environment}
        \label{fig:sub1}
    \end{subfigure}
    \hfill
    \begin{subfigure}{0.32\linewidth}
        \centering
        \includegraphics[width=\linewidth]{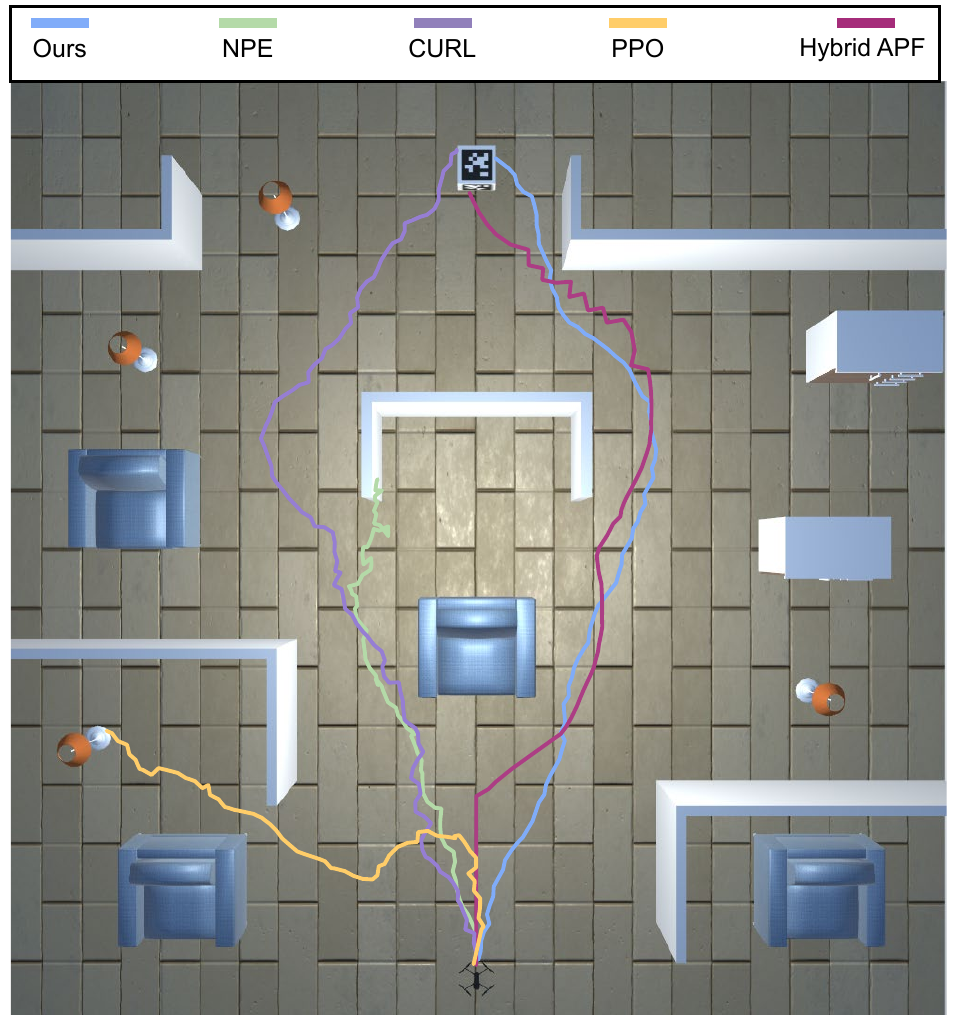}
        \caption{Barrier environment}
        \label{fig:sub2}
    \end{subfigure}
    \hfill
    \begin{subfigure}{0.32\linewidth}
        \centering
        \includegraphics[width=\linewidth]{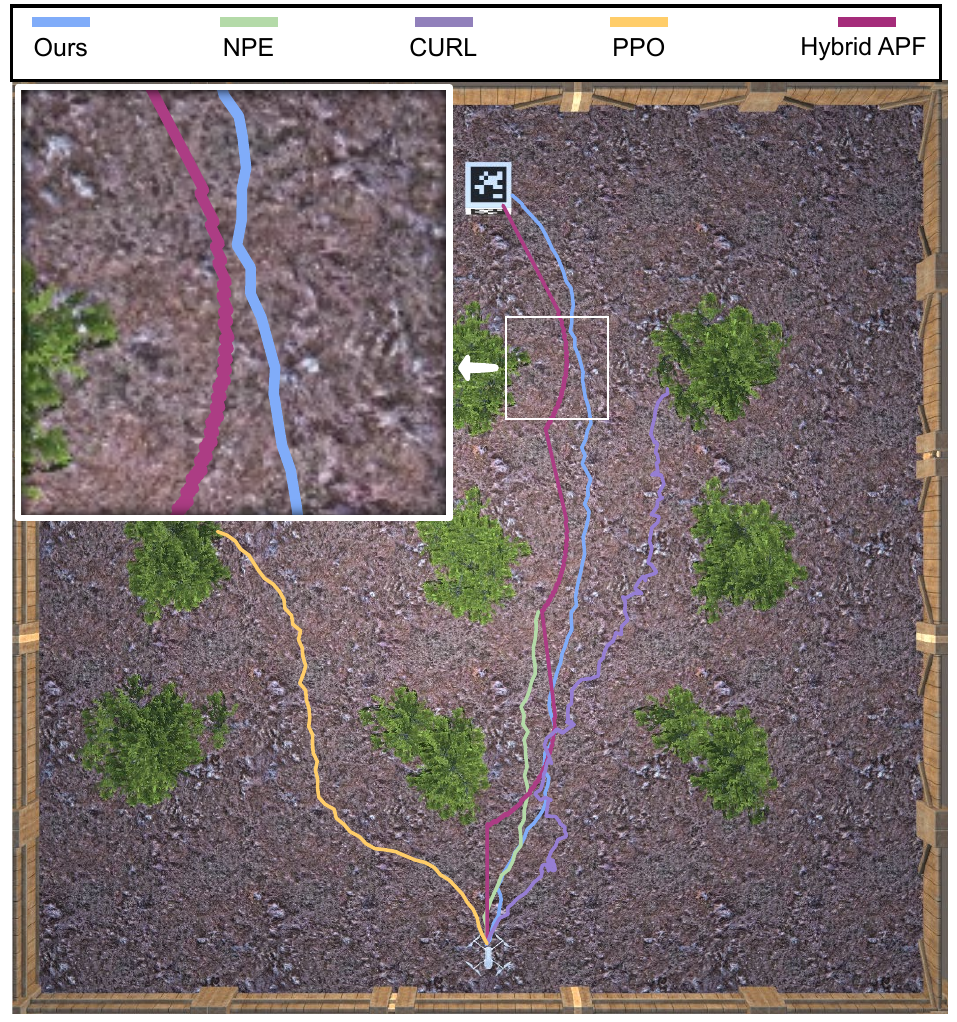}
        \caption{Tree environment}
        \label{fig:sub3}
    \end{subfigure}

    \caption{Qualitative trajectory comparisons of baselines in three evaluation environments. The results demonstrate that the oracle and OMC-RL consistently generate smooth and efficient trajectories. In contrast, other baselines often produce suboptimal or erratic trajectories and tend to fail in environments with irregular layouts and complex textures.}
    \label{fig:tra}
\end{figure*}

\begin{table*}[htbp]
\centering
\setlength{\tabcolsep}{4pt} 
\renewcommand{\arraystretch}{1.1}
\begin{threeparttable}
\caption{Performance of Baselines in Simulation Environments}
\label{tab:baseline}
\begin{tabular*}{\textwidth}{@{\extracolsep{\fill}}lcccccccccccccccccc}
\toprule
\multirow{2}{*}{Method} & 
\multicolumn{6}{c}{Furniture Scenario (Easy)} & 
\multicolumn{6}{c}{Barrier Scenario (Medium)} & 
\multicolumn{6}{c}{Tree Scenario (Hard)} \\
\cmidrule(lr){2-7} \cmidrule(lr){8-13} \cmidrule(lr){14-19}
& NE$\downarrow$ & OS$\uparrow$ & SR$\uparrow$ & SPL$\uparrow$ & CR$\downarrow$ & TTS$\downarrow$
& NE$\downarrow$ & OS$\uparrow$ & SR$\uparrow$ & SPL$\uparrow$ & CR$\downarrow$ & TTS$\downarrow$
& NE$\downarrow$ & OS$\uparrow$ & SR$\uparrow$ & SPL$\uparrow$ & CR$\downarrow$ & TTS$\downarrow$ \\
\midrule
Ours & \textbf{1.49} & \textbf{92.5} & \textbf{91.5} & \textbf{0.88} & \textbf{2.5} & \textbf{502}
              & \textbf{1.62} & \textbf{86.0} & \textbf{84.5} & \textbf{0.80} & \textbf{5.5} & \textbf{525}
              & \textbf{2.01} & \textbf{80.5} & \textbf{78.5} & \textbf{0.73} & \textbf{8.0} & \textbf{516} \\
NPE~\cite{zhang2024npe} & 6.08 & 21.5 & 17.5 & 0.12 & 80.5 & 564
                        & 7.59 & 9.5  & 0.0  & 0.00 & 94.5 & --
                        & 7.83 & 9.0  & 0.0  & 0.00 & 97.0 & -- \\
CURL~\cite{laskin2020curl} & 2.07 & 75.0 & 72.5 & 0.68 & 22.0 & 589
                           & 5.07 & 34.0 & 30.0 & 0.26 & 62.0 & 603
                           & 7.36 & 12.0 & 0.0  & 0.00 & 79.5 & -- \\
PPO~\cite{schulman2017proximal} & 6.90 & 17.5 & 12.0 & 0.08 & 87.0 & 596
                                & 7.36 & 12.0 & 0.0  & 0.00 & 90.0 & -- 
                                & 7.78 & 9.5  & 0.0  & 0.00 & 98.0 & -- \\
Hybrid APF~\cite{pan2021improved} & 2.56 & 73.5 & 72.0 & 0.69 & 7.0 & 741
                                  & 5.23 & 31.0 & 28.5 & 0.25 & 17.5 & 789
                                  & 3.04 & 71.5 & 69.5 & 0.66 & 10.5 & 753 \\
\bottomrule
\end{tabular*}
\vspace{1mm}
\footnotesize *Bold values indicate the best results. “--” indicates zero success.
\end{threeparttable}
\end{table*}

\begin{figure*}[!t]
   \centering
   \includegraphics[width=1.0\textwidth]{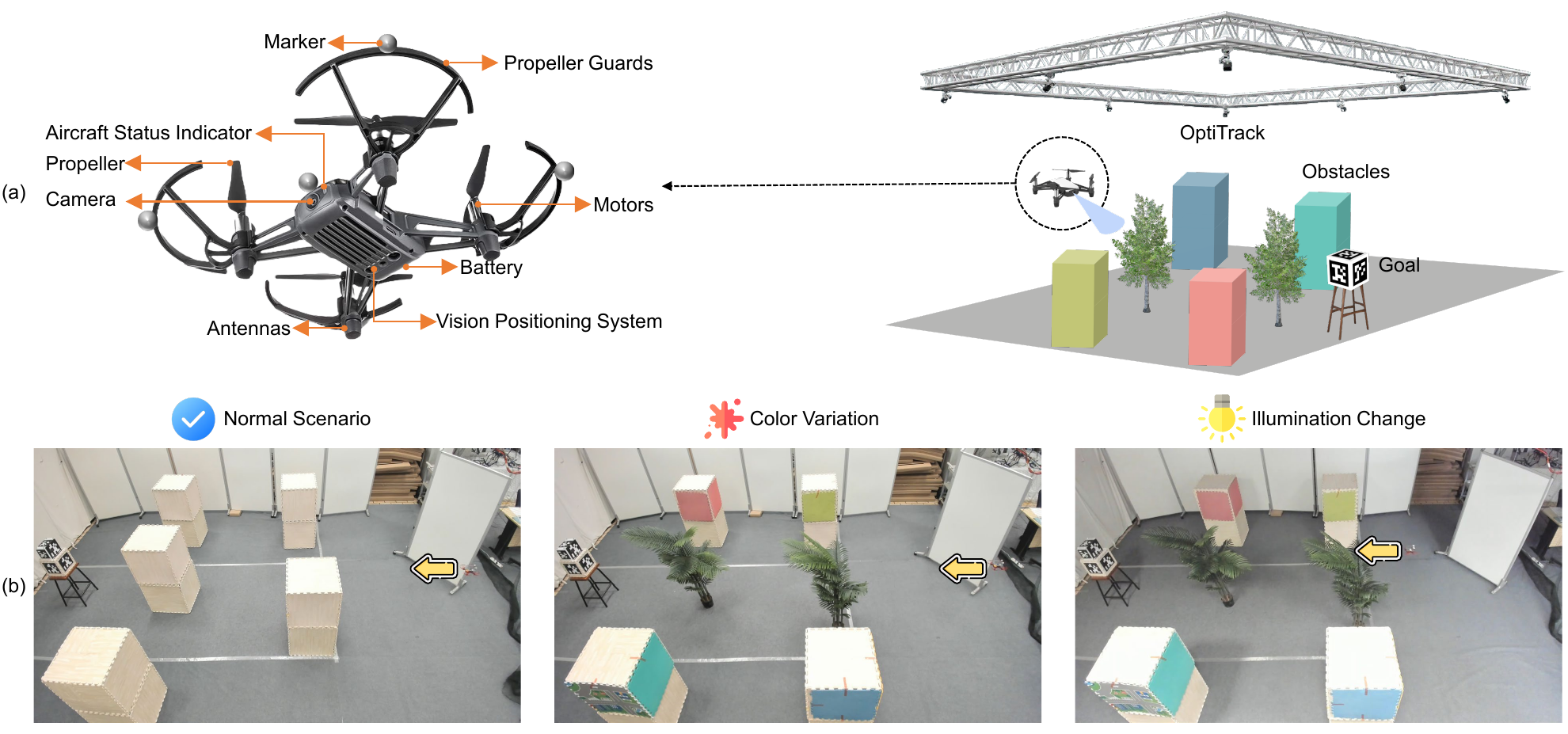} 
   \caption{Real-world deployment setup and evaluation scenarios. (a) Hardware components of the DJI Tello Edu platform and schematic illustration of the indoor test environment. The drone is equipped with a monocular RGB camera and receives localization updates from an OptiTrack motion capture system. (b) Three evaluation scenarios designed to assess robustness under visual domain shifts: normal scenario with uniform obstacle appearance, color variation introducing obstacles with diverse textures and colors, and illumination change induced by manually toggling ambient lighting conditions during flight.}
   \label{fig:setup}
\end{figure*}

\begin{figure*}[htbp]
    \centering
    \subfloat[Visual comparison of flight trajectories in three physical scenarios.]{%
        \includegraphics[width=0.9\linewidth]{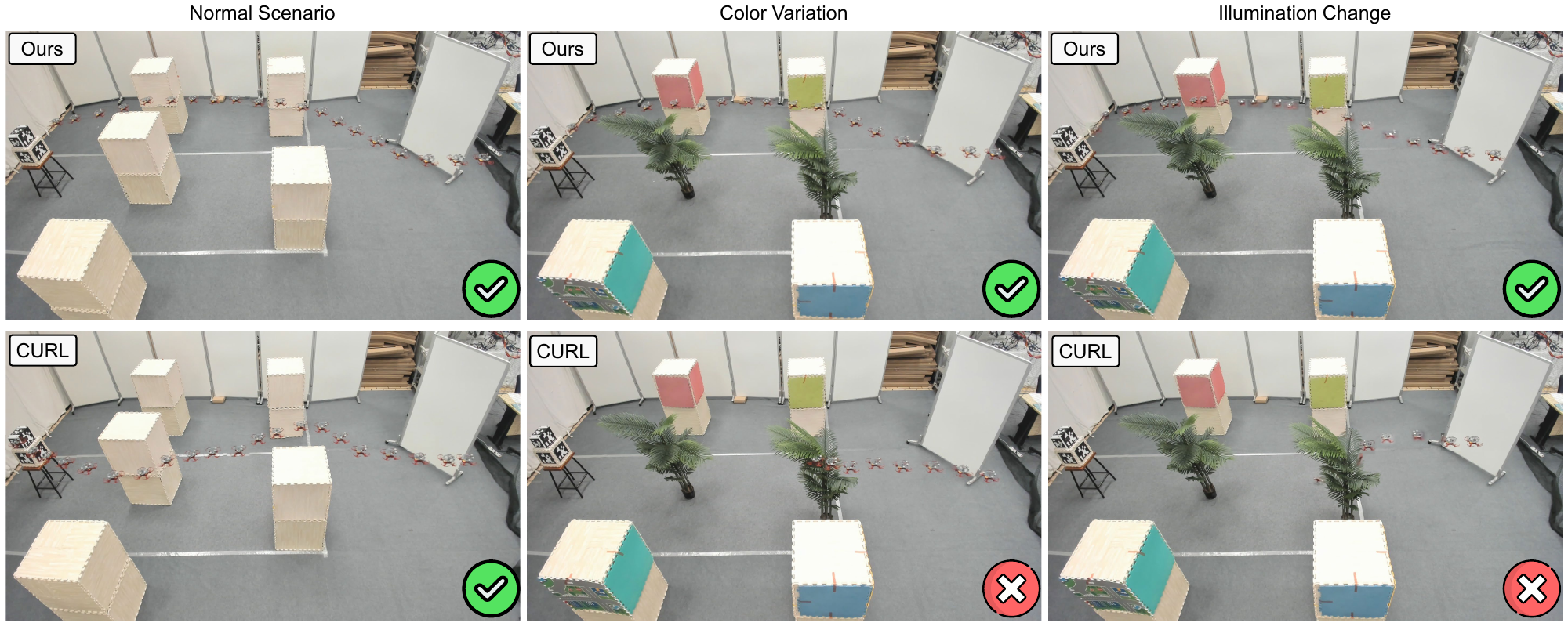}
        \label{fig:top}
    }\vspace{2mm}

    \subfloat[Velocity-annotated trajectories highlighting movement patterns of OMC-RL and CURL.]{%
        \includegraphics[width=0.9\linewidth]{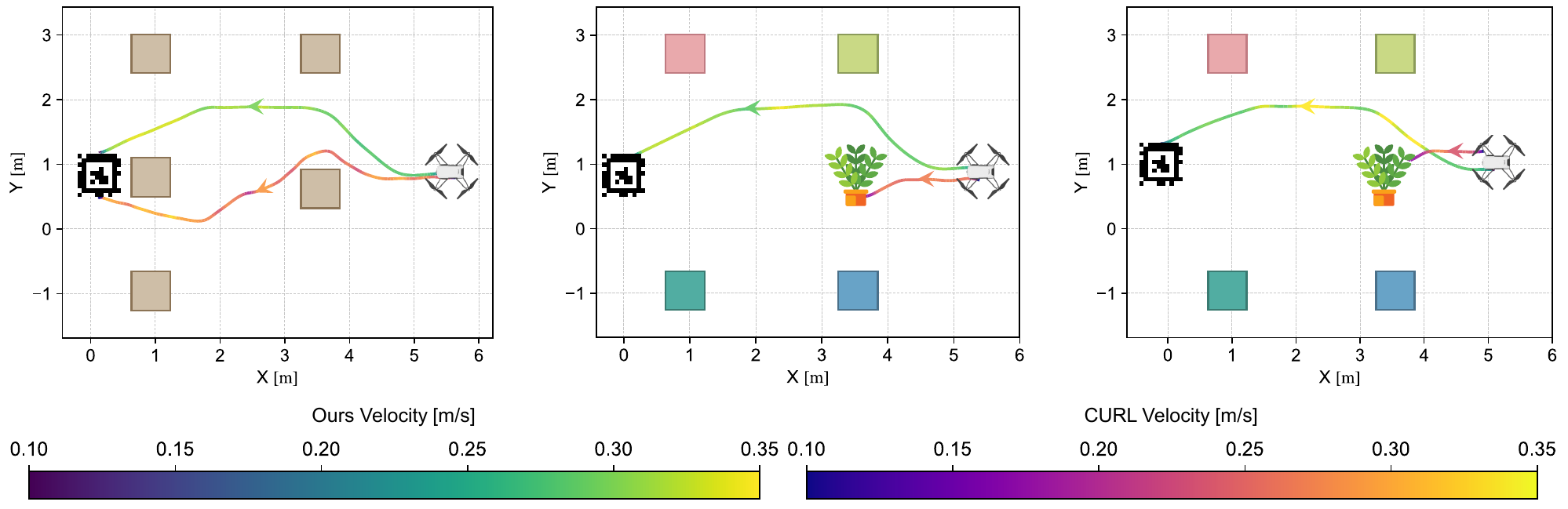}
        \label{fig:bottom}
    }

    \caption{Real-world trajectory evaluation under visual domain shifts. Our method consistently produces smooth and successful trajectories across all scenarios, demonstrating strong robustness and generalization. In contrast, CURL frequently fails when facing visual disturbances such as obstacle color variation, texture inconsistency, and sudden illumination shifts.}
    \label{fig:tra_physical}
\end{figure*}

\begin{figure}[htbp]
    \centering
    \includegraphics[width=1.0\linewidth]{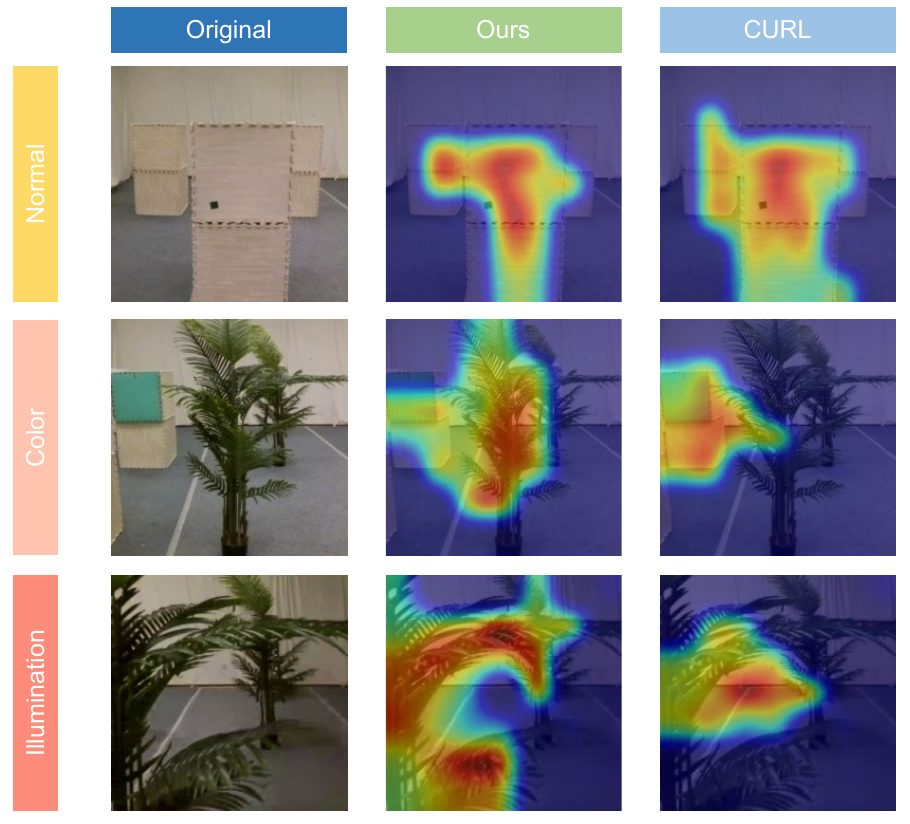}
    \caption{Attention visualization using Grad-CAM \cite{selvaraju2017grad} in physical environments. Our method consistently focuses on task-relevant regions across normal, color-shifted, and illumination-varied scenes, whereas CURL shows misplaced attention, especially under visual disturbances.}
    \label{fig:attention}
\end{figure}

\begin{figure}[htbp]
    \centering
    \includegraphics[width=1.0\linewidth]{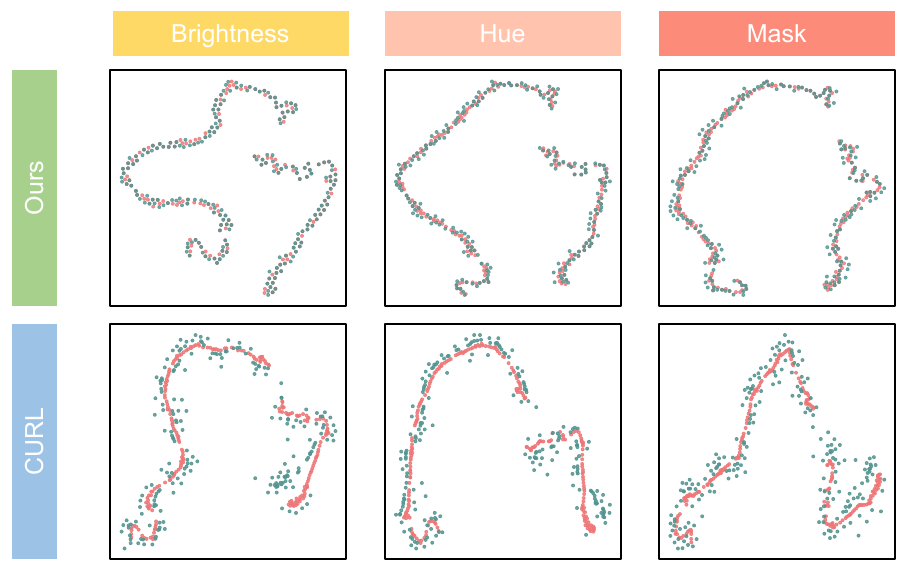}
    \caption{t-SNE \cite{van2008visualizing} visualization of learned features under input disturbances. Our method maintains compact and consistent feature clusters under brightness, hue, and masking conditions, while CURL exhibits scattered and misaligned representations.}
    \label{fig:tsne}
\end{figure}

\subsection{Real-World Experiments}

To further evaluate the effectiveness and generalizability of OMC-RL, we conduct real-world experiments in two indoor environments. The real-world deployment setup is illustrated in Fig. \ref{fig:setup}. The first is a simple indoor scene containing multiple box-shaped obstacles with an identical appearance. The second is a more complex setting that introduces additional perceptual challenges, including potted plants and obstacles with diverse textures and colors, which increase scene diversity and pose greater sim-to-real domain shifts for visual perception. In addition, to simulate abrupt illumination changes often encountered in real deployments, we manually switch ambient lighting on and off during flight to test robustness against visual perturbations. The drone used for testing is a DJI Tello Edu, a lightweight quadrotor equipped with a forward-facing monocular RGB camera that streams $720p$ video at $30$ FPS and provides a field of view (FOV) of $82.6^\circ$. It receives localization updates at $120$ Hz via the OptiTrack motion capture system. Start and goal positions are manually specified before each flight. Obstacles are perceived using the onboard RGB stream, which is resized and processed through the pretrained feature encoder to extract compact representations. These representations, along with other state information, are then fed into the downstream policy network to generate continuous control commands. The entire model is exported in ONNX format and deployed for real-time inference without extra fine-tuning.

We evaluate the model under 20 trials per environment, across varying start-goal configurations. The method achieves a success rate of $70\%$ in the normal scene, $65\%$ under color variations, and $55\%$ under illumination disturbance. In comparison, CURL only achieves $45\%$ in the normal scene and fails entirely under color and illumination disturbances. Flight trajectories for representative runs are visualized in Fig. \ref{fig:tra_physical}. Despite the presence of visual distractors such as non-uniform obstacle textures, cluttered plant structures, and sudden lighting changes, OMC-RL demonstrates robust and generalizable performance. This validates the strength of our masked contrastive learning framework, which enables the encoder to focus on temporally coherent and task-relevant representations rather than overfitting to domain-specific visual patterns, thereby enhancing real-world transferability.

To further validate the effectiveness of OMC-RL in real-world scenarios, we conduct qualitative analyses on visual attention and feature consistency under various disturbances. As illustrated in Fig. \ref{fig:attention}, we visualize the attention heatmaps generated by OMC-RL and CURL using Grad-CAM \cite{selvaraju2017grad} across three physical scenarios. Compared to CURL, OMC-RL consistently focuses on salient objects of interest, even in the presence of color variations and illumination changes, indicating improved robustness to perceptual noise and stronger task awareness. We further assess the representation stability of OMC-RL by analyzing the encoder’s response to disturbed inputs. Specifically, we collect FOV images from successful trajectories and apply random changes in brightness and hue, as well as random masking over inputs. These modified observations are passed through the feature encoders of OMC-RL and CURL, and the resulting features are projected using t-SNE \cite{van2008visualizing}, as shown in Fig. \ref{fig:tsne}. The features produced by OMC-RL remain densely clustered and well-aligned with those from the original inputs, indicating strong invariance and continuity in its learned representations. In contrast, features from CURL exhibit noticeable shift and dispersion under the same conditions, reflecting higher sensitivity to domain shifts and reduced generalization.

\subsection{Ablation Studies}

To better interpret the contributions of each component in OMC-RL, we conduct a series of ablation studies and analyses. Note that all ablation results are reported across evaluation environments, as our primary focus lies in assessing the model's generalizability beyond the training domain.

\begin{figure}[htbp]
    \centering
    \includegraphics[width=1.0\linewidth]{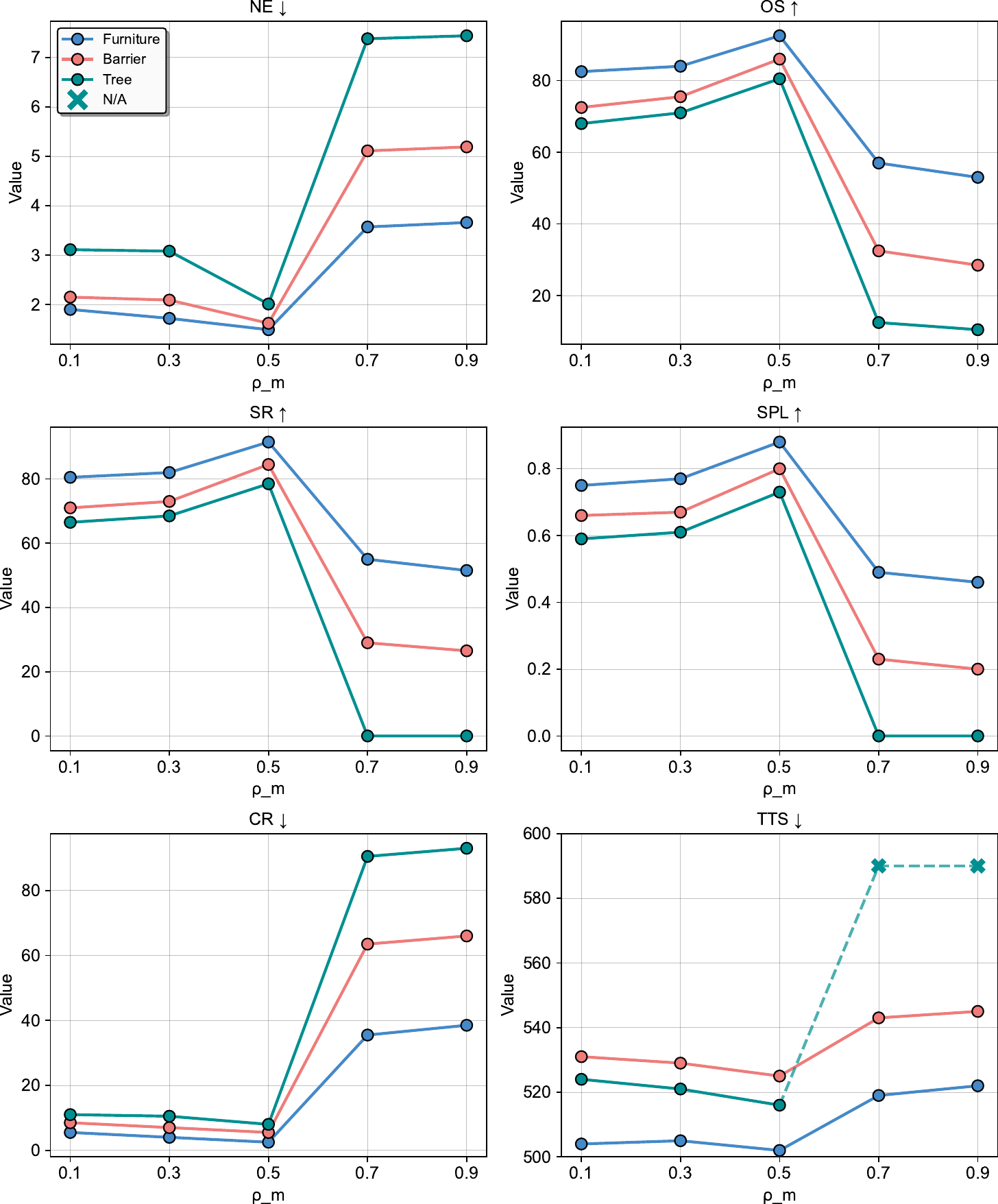}
    \caption{Ablation study on masking probability \(\varrho_m\). We evaluate downstream navigation performance with encoders pretrained under different masking probabilities. Navigation metrics are visualized as line plots across three simulation environments. A masking ratio (\(\varrho_m = 0.5\)) achieves the best overall performance. In contrast, overly low or high values lead to performance degradation due to insufficient use of the Transformer’s reconstruction capability or excessive information loss.}
    \label{fig:mask_probability}
\end{figure}

\begin{table}[htbp]
\centering
\renewcommand{\arraystretch}{1.1}
\setlength{\tabcolsep}{3.5pt} 
\begin{threeparttable}
\caption{Ablation on Sequential Input for Contrastive Representation Learning}
\label{tab:sequential}
\begin{tabular}{cccccccc}
\toprule
Environment & Method & NE$\downarrow$ & OS$\uparrow$ & SR$\uparrow$ & SPL$\uparrow$ & CR$\downarrow$ & TTS$\downarrow$ \\
\midrule
\multirow{2}{*}{Furniture (Easy)} 
    & CURL      & \textbf{2.07} & \textbf{75.0} & \textbf{72.5} & \textbf{0.68} & \textbf{22.0} & \textbf{589} \\
    & CURL-cons & 5.04          & 33.0          & 30.5          & 0.24          & 62.0          & 610 \\
\midrule
\multirow{2}{*}{Barrier (Medium)}   
    & CURL      & \textbf{5.07} & \textbf{34.0} & \textbf{30.0} & \textbf{0.26} & \textbf{62.0} & \textbf{603} \\
    & CURL-cons & 7.09          & 14.5          & 0.0           & 0.00          & 84.0          & -- \\
\midrule
\multirow{2}{*}{Tree (Hard)}      
    & CURL      & \textbf{7.36} & \textbf{12.0} & 0.0           & 0.00          & \textbf{79.5}          & -- \\
    & CURL-cons & 7.97          & 7.5           & 0.0           & 0.00          & 89.5          & -- \\
\bottomrule
\end{tabular}
\end{threeparttable}
\end{table}

\textbf{(1) Masking Probability.} We investigate the impact of varying the masking probability \(\varrho_m\) on downstream navigation performance. Specifically, we pretrain separate encoders under different \(\varrho_m \in \{0.1,\ 0.3, \ 0.5,\ 0.7,\ 0.9\}\), and pair each encoder with a newly initialized policy network trained in the box-based environment. The resulting policies are then tested in evaluation environments. The corresponding metrics are summarized in Fig. \ref{fig:mask_probability}. As shown, setting \(\varrho_m = 0.5\) achieves the best overall performance. A high masking ratio (e.g., \(\varrho_m = 0.9\) or \(\varrho_m = 0.7\)) leads to a significant performance drop, even underperforming CURL, as excessive masking removes most spatial and semantic cues from the input, impairing the model's ability to infer latent structure and recover useful features from corrupted inputs. Similarly, a low masking ratio (e.g., \(\varrho_m = 0.1\) or \(\varrho_m = 0.3\)) also hampers generalization. According to Eq. \eqref{cl loss}, insufficient masked positions result in sparse training signals, making it difficult for the model to learn useful temporal correlations and reconstruct representations. These results confirm the utility of our masked contrastive learning framework. Properly tuned masking probability encourages the encoder to extract consistent and temporally predictive features, while both under-masking and over-masking degrade downstream navigation performance due to insufficient supervisory signals or excessive information loss.

\begin{figure}[htbp]
    \centering
    \includegraphics[width=1.0\linewidth]{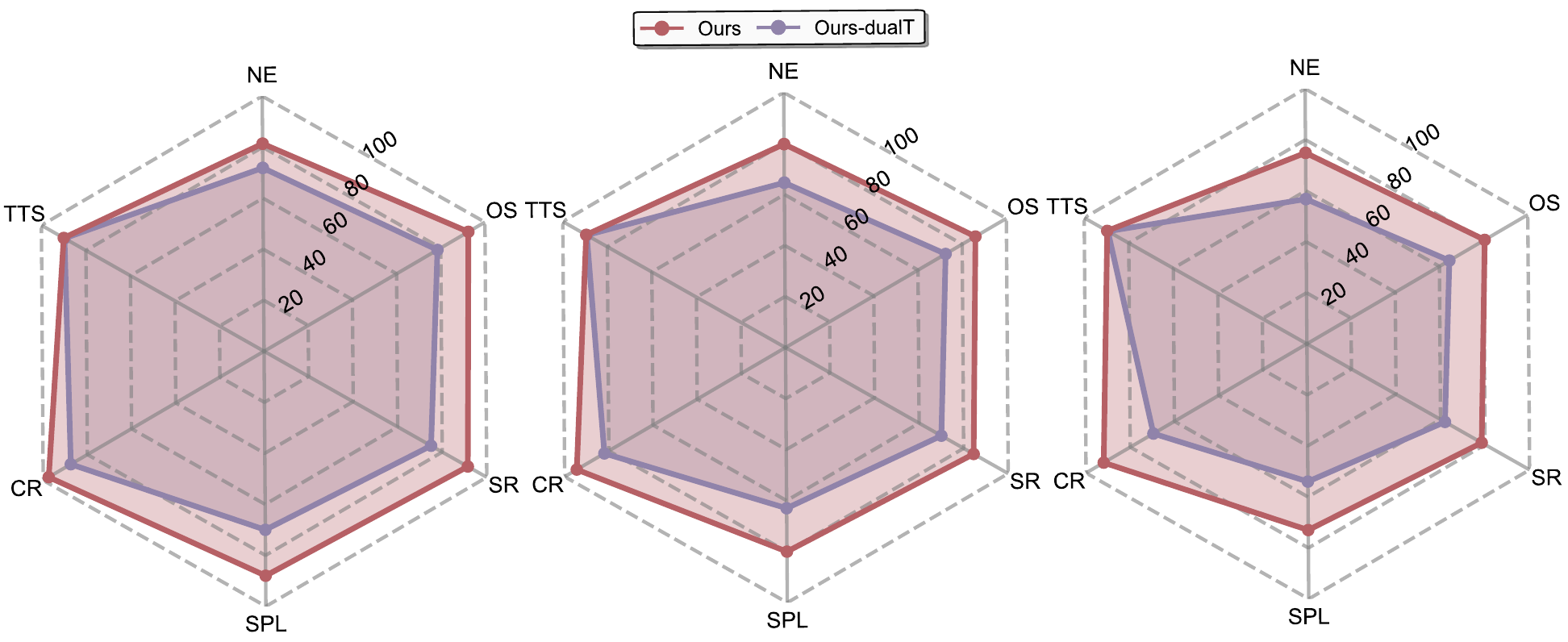}
    \caption{Ablation study on the Transformer module design. We compare our setup (Transformer on query encoder only) with a dual-Transformer variant where both query and key encoders use Transformers. Radar plots visualize normalized navigation metrics, where higher values indicate better performance. The dual-Transformer design degrades performance due to over-reliance on temporal modeling, weakening the CNN encoder’s standalone capability during downstream policy learning.}
    \label{fig:transformer}
\end{figure}

\begin{figure}[t!]
    \centering
    \includegraphics[width=0.95\linewidth]{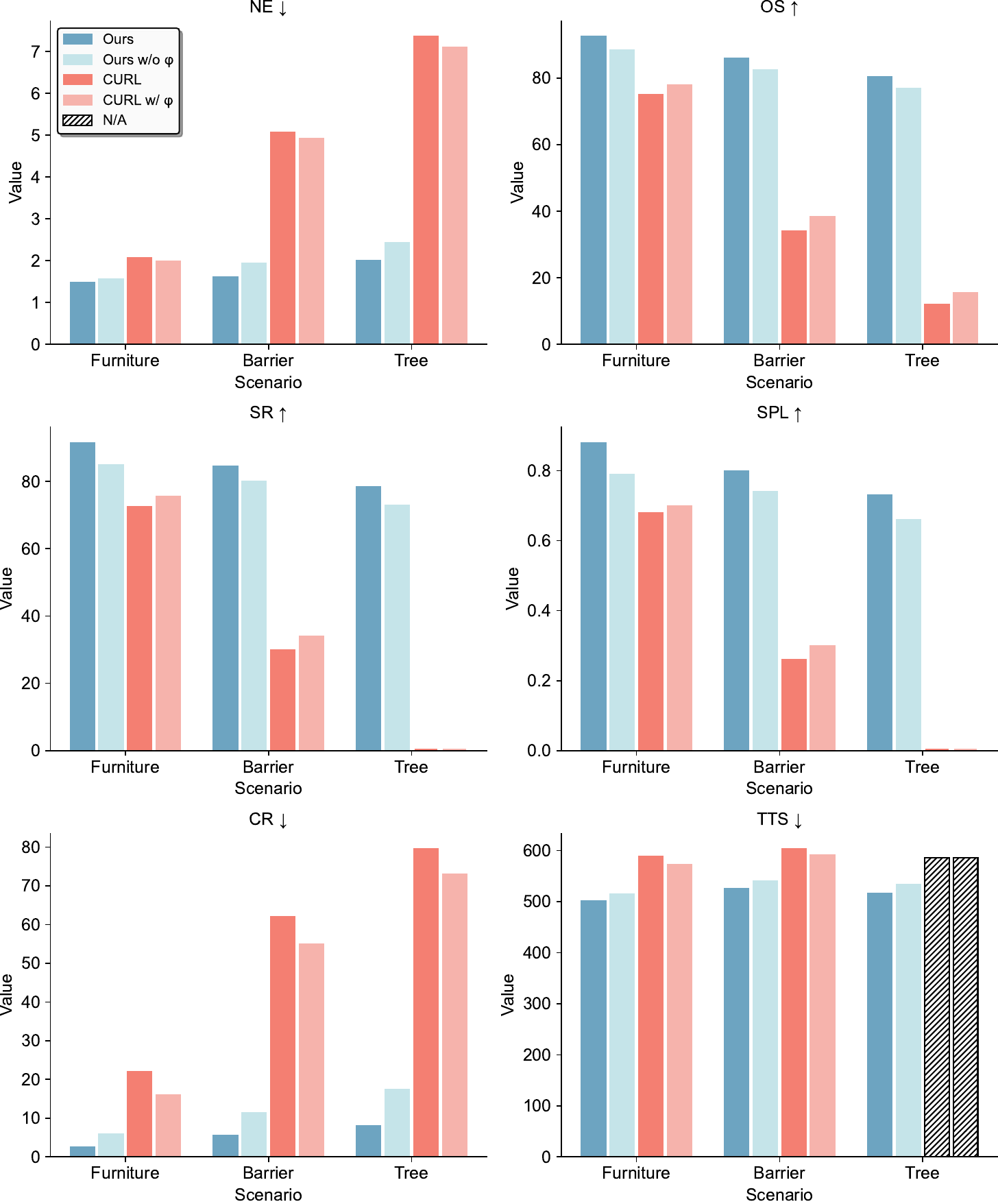}
    \caption{Ablation study on the non-linear projection module \(\varphi\). We visualize downstream navigation performance using bar plots. Results demonstrate that removing the projection head leads to consistent degradation compared to the full model. While adding \(\varphi\) to CURL improves its performance, both CURL variants remain inferior to our approach. These results highlight the importance of \(\varphi\) in learning representations.}
    \label{fig:projection}
\end{figure}

\begin{figure*}[t!]
    \centering
    \includegraphics[width=1.0\linewidth]{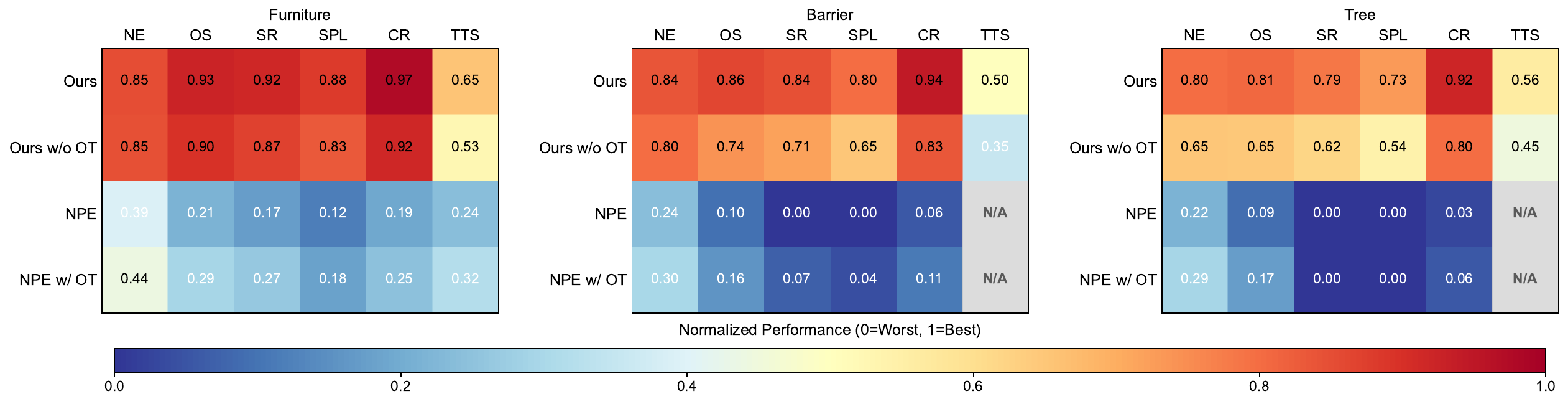}
    \caption{Ablation Study on Oracle Teacher (OT). The heatmap visualizes the normalized performance of all metrics, where higher values indicate better performance. Removing the oracle teacher leads to moderate drops in the furniture scene, but causes significant degradation in more complex settings. This highlights the oracle's contribution in guiding policy training. Adding OT to the NPE baseline yields improvements over NPE, but still falls short of our method. These results highlight the benefit of oracle guidance for learning effective policies.}
    \label{fig:heatmap}
\end{figure*}

\textbf{(2) Sequential Input} In our method, $T$ temporally consecutive frames are sampled from $\mathcal{H}$ and processed by the Transformer, where contrastive pairs are constructed through masking. In contrast, CURL forms contrastive pairs by augmenting individual frames independently sampled from the dataset. To investigate the effect of sequential input, we modify CURL to operate on consecutive frames rather than random ones, denoted as CURL-cons. Empirically, we observe that training CURL-cons with high-frequency sequential frames leads to severe instability and even encoder collapse. Lowering the sampling frequency improves training stability but still results in significantly degraded policy performance, as shown in Table \ref{tab:sequential}. These findings suggest that simply applying sequential input to CURL introduces visually redundant and highly correlated features that are difficult to differentiate using basic data augmentations, thereby limiting the effectiveness of contrastive learning and downstream generalization. The results further confirm that the performance gains of our method are not solely due to temporal continuity, but also depend on the use of masking during sequence modeling.

\textbf{(3) Transformer Module.} Our method adopts an asymmetric design in which only the query encoder is equipped with a Transformer module, while the key branch remains purely convolutional. Although this avoids additional computational overhead, we further investigate whether this asymmetry is indeed beneficial through a targeted ablation. We introduce an additional Transformer module $\xi_k$ for the key encoder, denoted as our-dualT. The parameters are updated using momentum contrastive learning following the paradigm in Section \ref{text:masked}. As shown in Fig. \ref{fig:transformer}, this design leads to a drop in performance. We attribute this to a mismatch between the contrastive learning objective and downstream usage: during policy learning, only the CNN encoder output is used to generate actions without the Transformer module. Introducing temporal modeling on both branches of masked contrastive learning encourages the model to rely heavily on the Transformer layers, thereby weakening the CNN encoder’s ability to encode temporal dependencies independently. This limits the utility of the learned representations once the Transformer is removed, and ultimately impairs policy generalization.

\begin{table*}[t!]
\centering
\setlength{\tabcolsep}{4pt}
\renewcommand{\arraystretch}{1.1}
\begin{threeparttable}
\caption{Ablation on Decay Strategies for Downstream Policy Learning}
\label{tab:decay}
\begin{tabular*}{\textwidth}{@{\extracolsep{\fill}}lcccccccccccccccccc}
\toprule
\multirow{2}{*}{Method} & 
\multicolumn{6}{c}{Furniture Scenario (Easy)} & 
\multicolumn{6}{c}{Barrier Scenario (Medium)} & 
\multicolumn{6}{c}{Tree Scenario (Hard)} \\
\cmidrule(lr){2-7} \cmidrule(lr){8-13} \cmidrule(lr){14-19}
& NE$\downarrow$ & OS$\uparrow$ & SR$\uparrow$ & SPL$\uparrow$ & CR$\downarrow$ & TTS$\downarrow$
& NE$\downarrow$ & OS$\uparrow$ & SR$\uparrow$ & SPL$\uparrow$ & CR$\downarrow$ & TTS$\downarrow$
& NE$\downarrow$ & OS$\uparrow$ & SR$\uparrow$ & SPL$\uparrow$ & CR$\downarrow$ & TTS$\downarrow$ \\
\midrule
Ours & \textbf{1.49} & \textbf{92.5} & \textbf{91.5} & \textbf{0.88} & \textbf{2.5} & \textbf{502}
              & \textbf{1.62} & \textbf{86.0} & \textbf{84.5} & \textbf{0.80} & \textbf{5.5} & \textbf{525}
              & \textbf{2.01} & \textbf{80.5} & \textbf{78.5} & \textbf{0.73} & \textbf{8.0} & \textbf{516} \\
Ours-fixed        & 1.77 & 84.5 & 81.5 & 0.75 & 6.5 & 517
                  & 1.84 & 79.0 & 76.0 & 0.69 & 9.5 & 538
                  & 2.69 & 72.5 & 71.0 & 0.63 & 14.0 & 528 \\
Ours-exp          & 1.65 & 87.0 & 84.5 & 0.77 & 5.5 & 512
                  & 1.90 & 81.0 & 78.5 & 0.71 & 8.0 & 533
                  & 2.56 & 75.5 & 72.5 & 0.65 & 13.0 & 524 \\
\bottomrule
\end{tabular*}
\end{threeparttable}
\end{table*}


\textbf{(4) Non-Projection Module.} Our method includes a non-linear projection module $\varphi$ prior to the contrastive loss. To assess its contribution to visuomotor policy learning, we conduct an ablation in which $\varphi$ is removed while keeping all other components unchanged, denoted as ours w/o~$\varphi$. The evaluation results, shown in Fig. \ref{fig:projection}, indicate that removing $\varphi$ leads to degraded policy performance. This probably stems from the reduced representational capacity of the feature encoder, as direct optimization without non-linear transformation can restrict the encoder from producing discriminative embeddings better for policy learning. While ours w/o~$\varphi$ underperforms the full model, it still exceeds our CURL baseline. Note that the vanilla CURL does not include the non-linear projection module. Here, a CURL variant incorporating $\varphi$ is used and denoted as CURL w/ $\varphi$, as explained in Section~\ref{text::baselines}. To further examine the effect of $\varphi$, we compare both CURL variants and observe that adding the projection head consistently improves performance. All aforementioned findings are aligned with the conclusion reported in \cite{chen2020simple}. However, both CURL variants remain inferior to our proposed method. These results support the utility of the non-linear projection module $\varphi$ and further highlight the role of masking and temporal correlation in learning transferable visuomotor representations.

\textbf{(5) Oracle Teacher.} In our method, an oracle teacher policy with access to full global state information is introduced to guide the early stage of downstream policy training. To evaluate the effectiveness of this guidance, we implement an ablation variant in which the oracle guidance is removed, denoted as our w/o OT. Furthermore, we replace the non-expert policy in the NPE baseline with the oracle teacher, denoted as NPE w/ OT. As summarized in Fig. \ref{fig:heatmap}, our w/o OT results in only a slight performance drop in the easy furniture environment. However, its performance degrades significantly in the more challenging barrier and tree settings. This demonstrates the importance of the oracle teacher, which leverages privileged information to generate ideal action commands. In mediated environments, the agent can still compensate through environmental exploration. However, as the scene becomes more complex and the visual appearance becomes more diverse and perceptually demanding, the absence of oracle guidance leads to more noticeable performance degradation. Results also show that NPE w/ OT outperforms the original NPE but remains inferior to both our w/o OT and our full model. These findings further confirm the benefit of oracle guidance and emphasize the significance of learning generalized representations for visuomotor policy learning.

\textbf{(6) Decay Schedule.} In our method, a decaying coefficient $\alpha$ is introduced in Eq. \eqref{total loss} to balance the RL objective and the KL divergence term. This coefficient is linearly annealed throughout training, allowing the policy to gradually transition from imitating the oracle to exploring the environment independently. To assess the impact of the decay schedule, we compare three variants: a fixed coefficient setting (denoted as our-fixed), an exponential decay setting (our-exp), and our proposed linear decay schedule. For our-exp, $\alpha$ is decayed by a factor of $0.95$ every $1000$ steps. As shown in Table \ref{tab:decay}, the linear decay consistently outperforms the other two variants. Our-fixed prevents the agent from sufficiently reducing its reliance on the oracle, thereby limiting late-stage exploration. Our-exp reduces $\alpha$ too rapidly in the early stage, resulting in insufficient imitation and limiting the agent’s ability to acquire meaningful guidance. In the later stage, the decay slows down excessively, hindering the agent's shift toward independent exploration. The linear schedule, by contrast, enables a smooth and progressive shift from imitation to exploration, facilitating better policy adaptation across environments.

\section{Conclusion}
In this work, we propose OMC-RL, a novel framework that combines masked contrastive learning with oracle-guided RL to improve the sample efficiency and asymptotic performance of visuomotor policies for drones. OMC-RL decouples upstream representation learning from downstream policy learning. For upstream, a Transformer-integrated module combining sequence modeling and masked contrastive learning is employed to extract temporally-aware and task-relevant representations from sequential visual inputs. After training, the Transformer component is discarded, while the visual encoder is frozen and used for downstream policy learning. During this stage, an oracle teacher policy, which has access to full global state information, supervises the agent in the early stage and gradually reduces its impact as the agent transitions to independent exploration. Extensive experiments in both simulated and real-world environments demonstrate that OMC-RL effectively improves sample efficiency and asymptotic performance. Moreover, it enhances policy generalization under various and challenging perceptual conditions. In future work, we plan to extend OMC-RL to other robotic vision tasks. Owing to its decoupled training of the feature encoder, the framework can capture task-agnostic visual features that support the learning of transferable policies across diverse settings. Another promising direction is to extend OMC-RL to multi-modal or instruction-guided policy learning scenarios, where grounding visual observations to semantic goals remains a major challenge.

\bibliographystyle{IEEEtran}
\bibliography{IEEEabrv,yuhang}

\vspace{-1cm}
\begin{IEEEbiography}[{\includegraphics[width=1in,height=1.25in,clip,keepaspectratio]{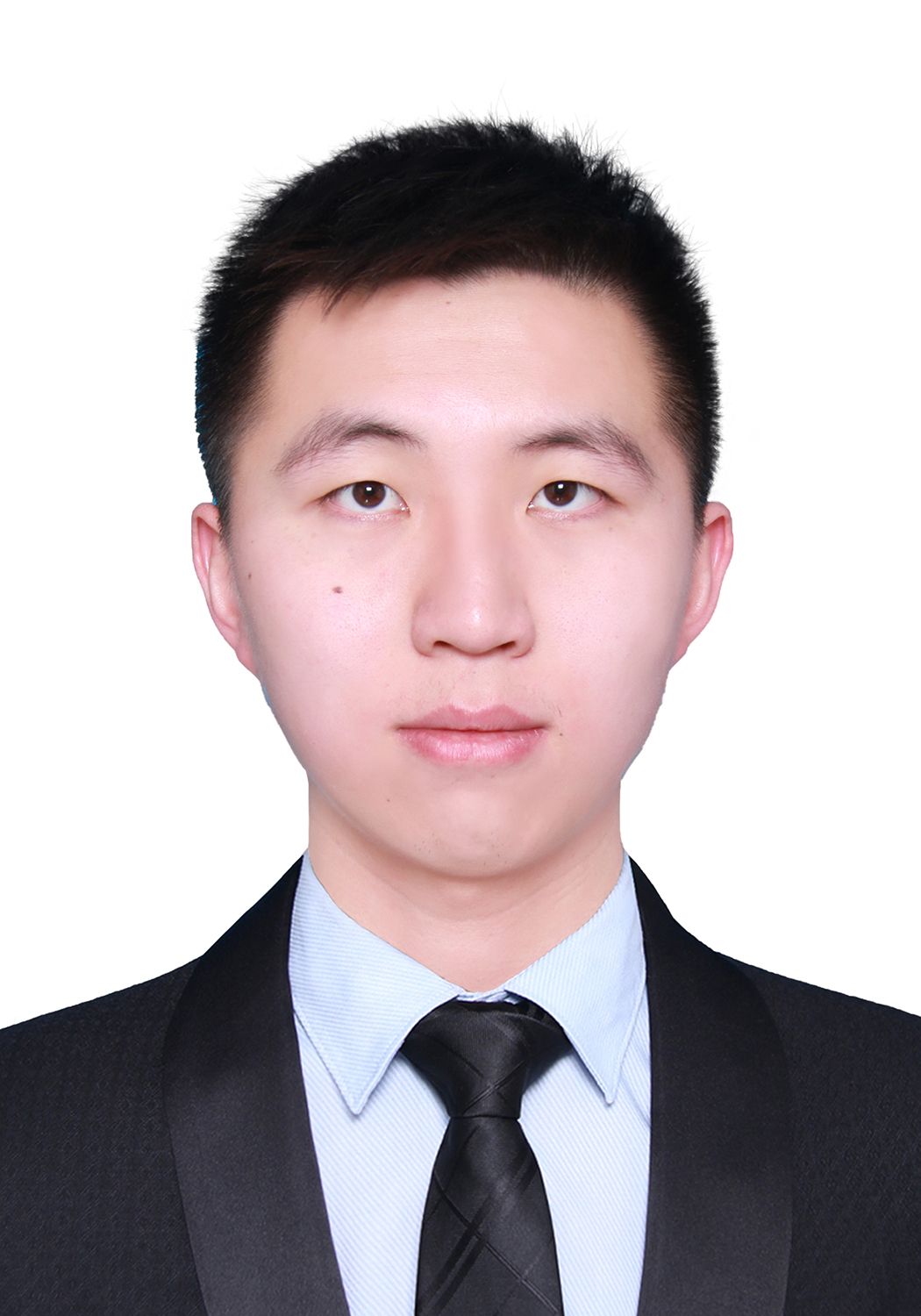}}]{Yuhang Zhang} (graduate student member, IEEE) received the B.E. degree in flight vehicle propulsion engineering from Harbin Engineering University, Harbin, China, in 2021 and the M.Eng. degree from the School of Mechanical \& Aerospace Engineering at Nanyang Technological University (NTU), Singapore, in 2023. Currently,  he is pursuing his Ph.D. degree at the School of Mechanical \& Aerospace Engineering at NTU, Singapore.

His research primarily focuses on unmanned aerial vehicles, deep reinforcement learning, and vision-and-language navigation.
\end{IEEEbiography}

\begin{IEEEbiography}[{\includegraphics[width=1in,height=1.25in,clip,keepaspectratio]{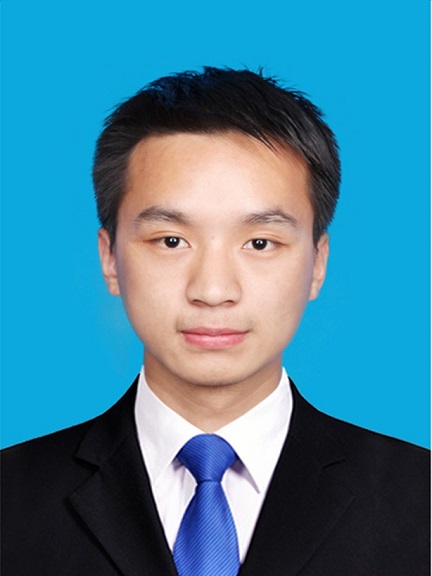}}]{Jiaping Xiao} (graduate student member, IEEE)
received the B.E. degree in aircraft design and engineering and the M.S. degree in flight dynamics and control from Beihang University, Beijing, China, in 2014 and 2017, and the Ph.D. degree in intelligent systems from Nanyang Technological University (NTU), Singapore, in 2024. He is currently a research fellow with the School of Mechanical and Aerospace Engineering, NTU.

His research interests include cyber-physical systems, reinforcement learning, machine vision, and aerial robotics.
\end{IEEEbiography}

\begin{IEEEbiography}[{\includegraphics[width=1in,height=1.25in,clip,keepaspectratio]{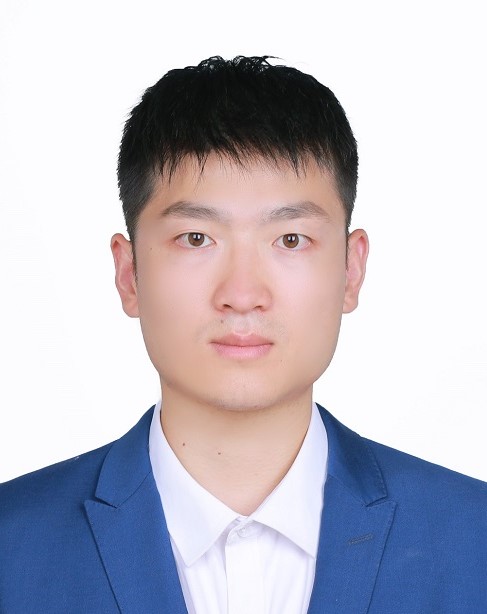}}]{Chao Yan} (member, IEEE) received the B.E. degree in electrical engineering and automation from China University of Mining and Technology, Xuzhou, China, in 2017, and the M.S. and Ph.D. degrees in control science and engineering from the National University of Defense Technology, Changsha, China, in 2019, and 2023, respectively. He was a visiting Ph.D. student with the School of Mechanical and Aerospace Engineering, Nanyang Technological University, Singapore, from 2021 to 2022.

He is currently an Associate Professor with the College of Automation Engineering, Nanjing University of Aeronautics and Astronautics, Nanjing, China. His research interests include deep reinforcement learning and coordination control of UAV swarms.
\end{IEEEbiography}

\begin{IEEEbiography}[{\includegraphics[width=1in,height=1.25in,clip,keepaspectratio]{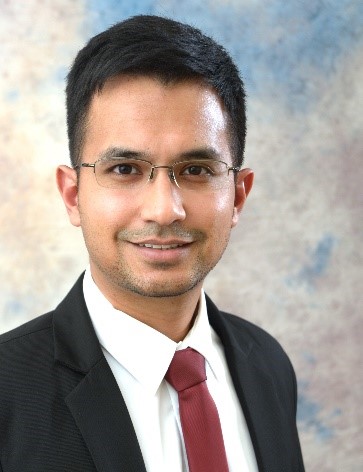}}]{Mir Feroskhan}
(member, IEEE) received B.E. degree (Hons.) in aerospace engineering from Nanyang Technological University, Singapore, in 2011, and the Ph.D. degree in aerospace engineering from the Florida Institute of Technology, Melbourne, FL, in 2016. He is currently an assistant professor with the School of Mechanical \& Aerospace Engineering at NTU.

His research interests include nonlinear control systems, multi-agent systems, flight dynamics and control, and aerial robotics.
\end{IEEEbiography}

\end{document}